\documentclass[sigconf]{acmart}

\usepackage[linesnumbered]{algorithm2e}

\usepackage{grffile} 
\usepackage{xcolor}
\definecolor{codegreen}{rgb}{0,0.6,0}
\definecolor{codepurple}{rgb}{0.58,0,0.82}

\usepackage{cleveref}
\usepackage{multirow}
\usepackage{listings}
\usepackage{setspace}

\usepackage{newfloat}
\DeclareFloatingEnvironment[
   fileext=lol,
   name=Listing
]{listing}

\usepackage{subcaption}
\DeclareCaptionSubType{listing}

\usepackage{mathtools} 

\DeclarePairedDelimiter{\ceil}{\lceil}{\rceil}

\newcommand{\afeat}{\text{f}}
\newcommand{\atag}{\text{t}}

\newcommand{\somematrix}{\ensuremath{M}}
\newcommand{\someothermatrix}{\ensuremath{\tilde M}}
\newcommand{\datamatrix}{\ensuremath{A}}

\newcommand{\datafeatmatrix}{\ensuremath{\datamatrix_\afeat}}
\newcommand{\datatagmatrix}{\ensuremath{\datamatrix_\atag}}

\newcommand{\Xmatrix}{\ensuremath{X}}
\newcommand{\Ymatrix}{\ensuremath{Y}}
\newcommand{\Yfeatmatrix}{\ensuremath{\Ymatrix_\afeat}}
\newcommand{\Ytagmatrix}{\ensuremath{\Ymatrix_\atag}}

\newcommand{\transpose}{\ensuremath{^\prime}}

\newcommand{\boolfield}{\ensuremath{F}}
\newcommand{\boolmissingfield}{\ensuremath{\boolfield_\missing}}
\newcommand{\missing}{\text{?}}

\newcommand{\nrank}{\ensuremath{k}}

\newcommand{\nx}{\ensuremath{m}}
\newcommand{\ndatasets}{\ensuremath{\nx}}

\newcommand{\ny}{\ensuremath{n}}
\newcommand{\ntags}{\ensuremath{\ny_\atag}}
\newcommand{\nfeatures}{\ensuremath{\ny_\afeat}}

\newcommand{\regstrength}{\ensuremath{\lambda}}

\newcommand{\sometag}{\ensuremath{\tau}}

\newcommand{\sizeof}[1]{\ensuremath{\left\vert#1\right\vert}}

\newcommand{\loss}{\ensuremath{\mathcal L}}

\copyrightyear{2020}
\acmYear{2020}
\setcopyright{acmcopyright}\acmConference[ICAIF '20]{ACM International Conference
on AI in Finance}{October 15--16, 2020}{New York, NY, USA}
\acmBooktitle{ACM International Conference on AI in Finance (ICAIF '20), October
15--16, 2020, New York, NY, USA}
\acmPrice{15.00}
\acmDOI{10.1145/3383455.3422537}
\acmISBN{978-1-4503-7584-9/20/10}

\begin{document}

\title{Paying down metadata debt: learning the representation of concepts using topic models}

\author{Jiahao Chen}
\orcid{0000-0002-4357-6574}
\author{Manuela Veloso}
\orcid{0000-0001-6738-238X}
\affiliation{
  \institution{J.\ P.\ Morgan AI Research}
  \city{New York}
  \state{New York}
}
\email{{jiahao.chen,manuela.veloso}@jpmorgan.com}


\begin{abstract}
We introduce a data management problem called metadata debt,
to identify the mapping between data concepts and their logical representations.
We describe how this mapping can be learned using semisupervised topic models
based on low-rank matrix factorizations that account for missing and noisy labels,
coupled with sparsity penalties to improve localization and interpretability.
We introduce a gauge transformation approach that allows us to construct explicit
associations between topics and concept labels, and thus assign meaning to topics.
We also show how to use this topic model for semisupervised learning tasks
like extrapolating from known labels, evaluating possible errors in existing labels,
and predicting missing features.
We show results from this topic model in predicting subject tags on over 25,000 datasets
from Kaggle.com, demonstrating the ability to learn semantically meaningful features.
\end{abstract}


\settopmatter{printfolios=true} 
\maketitle

\section{Introduction}
\label{sec:intro}

``What data do I have?'' is a surprisingly difficult question to answer
when managing data at the scale of large organizations \citep{Khatri2010}.
A complete answer requires not just the physical location of data (what database or filesystem?),
but also the corresponding logical representation (what attributes and relations?)
and conceptual formulation (what business semantics?) \citep{West2003}.
Understanding the relationship between these successive levels of abstraction is
necessary in order to define data governance needs such as data quality rules
\citep{Wang1996}.

As a simple example, suppose we need to enforce a rule that
the balance on the customer's account is positive.
To define a data quality check, we need to know to look up the table \verb|ACCOUNTS|, find the row that matches the current account ID in the column \verb|ACCT_ID|,
and check that the corresponding entry in the \verb|ACCT_BAL| column is not missing and is a positive number.
Translating the conceptual business need into a computable statement thus requires explicit knowledge
of the semantic concepts in the data (``account'' and ``balance'') and
the logical representation of that data (which tables and columns).
Having a complete inventory of concepts and their logical representations
is therefore a key component of data governance,
even being a regulatory requirement for compliance
with Basel standards \citep{bcbs239} and others.

In practice, however, organizations often lack the metadata
needed to describe the conceptual--logical mapping we have just described.
For example, collaborative tagging from crowdsourced data quickly yields
folksonomies of conceptual tags, but they may be inconsistent due to
personal biases and therefore require standardization \citep{Guy2006}.
Other situations arise in legacy enterprises,
where the merger of multiple information systems results in 
multiple logical representations of the same data concept,
or when loss of informal knowledge or incomplete documentation
results in an incompletely documented system,
or when a legacy system is repurposed for a new business purpose,
whose conceptual needs are not perfectly aligned with the existing
logical layout of data.
Furthermore, the cost of correctly assigning concept labels
can be very high when the semantic definition of the concepts
is couched in technical jargon which requires subject matter
expertise to comprehend.
Similarly, similar concepts may require expertise to disambiguate
and label correctly.
The challenge of correct conceptual labeling is exacerbated when the controlled vocabulary
evolves over time, a phenomenon known as concept drift \citep{Wang2011},
or when the logical representation of data also changes, as APIs and data formats change \citep{Gama2014}.
Therefore, it is possible that despite the best intentions and efforts of
organizations, that they still struggle to maintain and verify that they have
a correct and up-to-date conceptual--logical mapping of data.

In this paper, we study the problem of \textit{metadata debt},
namely the technical debt incurred when an organization lacks a complete and current
mapping between its data concepts and their corresponding logical representations.
We propose in \Cref{sec:topicmodel} to use semisupervised topic modeling to learn the conceptual--logical mapping of data,
making full use of existing concept labels, be they incomplete or noisy.
We first describe how an unsupervised topic model like latent semantic indexing (LSI)
contains a hidden invariance that we can exploit to write down an unambiguous,
gauge-generalized variant (GG-LSI) for the semisupervised setting in \Cref{sec:lsi,sec:gglsi}.
Next, we show how a more sophisticated semisupervised topic model can be expressed in
the formalism of generalized low-rank models (GLRMs) in \Cref{sec:glrm}, 
which allows us to specify a combination of elementwise logistic losses with customizable sparsity penalties.
We illustrate with a simple example how to interpret the results of all these methods in \Cref{sec:example,sec:example2}.
Next, we describe in \Cref{sec:subsampling} how to use the numerical trick of randomized subsampling to greatly reduce the computational cost of
topic modeling, and GLRMs in particular, so that the compute time is asymptotically linear in the number of topics.
This linear-time trick allows us to compute GLRMs efficiently and analyze a large folksonomy on Kaggle.com in \Cref{sec:kaggle}.
We introduce related work in each relevant section,
to better highlight how work in data management,
topic modeling, and numerical linear algebra come together to make this study possible.



\section{Semisupervised learning from sparse topic models}
\label{sec:topicmodel}

In this section, we present a matrix factorization approach which is explicitly constructed to yield interpretable topic models.
The relationship between topic models and matrix factorizations is well understood \citep{Donoho2003,Arora2012};
our contribution here is a semisupervised topic model which has the following features:
\begin{enumerate}
\item
The model uses a ternary encoding that distinguishes between present entries (1), missing entries (?) and truly absent entries (0)
for both features and labels.

\item
The model is tolerant of label noise and missingness.

\item
The model assigns explicit meaning to topics using anchor labels.

\item
The model admits a matrix factorization that can be made more interpretable through sparsity penalties.

\item
The model admits the specification of an arbitrary prior on unknown entries of the matrix factorization.

\item
The model has an inbuilt quantification of uncertainty.
\end{enumerate}

Previous work has argued for the importance of sparsity in topic models \citep{Zhu2011},
and studied uncertainty quantification for label noise \citep{Northcutt2019},
albeit separately.
To our knowledge, our work is the first semisupervised topic modeling approach that features all of the above.

We begin by stating the set of values we build our observation matrix from,
and introduce some definitions that we will use to set up our theoretical framework.
Our matrix factorization model is formulated over a generalization of Boolean values that
also includes an explicit representation of missing values, $\missing$.
We denote this set as $\boolmissingfield = \{0, 1, \missing\}$, and distinguish it from
the set of ordinary Boolean values $\boolfield = \{0, 1\}$%
\footnote{Unlike Boolean values $\boolfield = \{0, 1\}$,
the set $\boolmissingfield$ does not form a field over Boolean addition and multiplication,
but instead has the weaker algebraic structure of a pre-semiring.
Despite the weaker structure, it is still possible to define matrix algebra and matrix factorizations,
in particular eigenvalue factorizations, over pre-semirings \citep{Gondran2008}.
We will ignore such technicalities in this paper.}.

\begin{definition}
An \textit{observed element} $x\in\boolmissingfield$ of $\boolmissingfield$ is an element such that $x \ne \missing$.
An \textit{unobserved element} or \textit{missing element} $x\in\boolmissingfield$ of $\boolmissingfield$ is an element such that $x = \missing$.
\end{definition}

\begin{definition}
\label{def:obs}
Let $\datamatrix\in\boolmissingfield^{m\times n}$ be a $m\times n$ matrix over $\boolmissingfield$.
Then, the \textit{observed set} of $\datamatrix$ is the set
\begin{equation}
\Omega_\datamatrix = \{(i,j)\vert\datamatrix_{i,j}\ne\missing\}
\end{equation}%
of coordinates for all non-missing matrix elements.
The \textit{observed set} of $\someothermatrix\in\boolfield^{m\times n}$ is simply the set of all valid coordinates 
$\Omega_\datamatrix = \{(i,j)\vert i\in\mathbb Z, j\in\mathbb Z, 1\le i\le m, 1\le j \le n\}$,
since by definition, all elements are observed.
\end{definition}

To simplify the notation, we will drop the subscript from $\Omega$ when the matrix referenced is clear from context.

\begin{definition}
\label{def:topicmodel}
Let $\datamatrix = \begin{pmatrix}\datatagmatrix & \datafeatmatrix\end{pmatrix}$ be a data matrix
where $\datatagmatrix \in \boolmissingfield^{\ndatasets \times \ntags)}$
and $\datafeatmatrix \in \boolfield^{\ndatasets \times \nfeatures}$.
The first $\ntags$ columns represent the presence or absence of tags,
and the remaining $\nfeatures$ columns representing the presence or absence of features.
Then, a \textit{labelled topic model} is the approximate matrix factorization
\begin{equation}
    \datamatrix = \begin{pmatrix}\datatagmatrix & \datafeatmatrix\end{pmatrix}
    \approx \Xmatrix\transpose\Ymatrix
    = \Xmatrix\transpose
    \begin{pmatrix} \Ytagmatrix& \Yfeatmatrix
    \end{pmatrix}
    \label{eq:topic-model}
\end{equation}
where $\Xmatrix\in\mathbb R^{\nrank \times \ndatasets}$ is a real $\nrank \times \ndatasets$ matrix,
$\Ytagmatrix\in\mathbb R^{\nrank \times \ntags}$ is a real $\nrank \times \ntags$ matrix,
$\Yfeatmatrix\in\mathbb R^{\nrank \times \nfeatures}$ is a real $\nrank \times \nfeatures$ matrix,
and $\Ymatrix = \begin{pmatrix} \Ytagmatrix& \Yfeatmatrix\end{pmatrix}$. 

\end{definition}

Our definition of a labelled topic model explicitly constrains missing values to be only in the matrix of tags, $\Ytagmatrix$.
For tags, we want to distinguish between missing and absent:
an absent tag means that the concept is definitely not present,
whereas a missing tag connotes some ambiguity that the tag could be assigned,
but no one had affirmed that the tag should definitely be present.
In contrast, missing features should be considered truly absent,
as they can be fully observed without ambiguity.
This formulation generalizes the usual form of topic modeling,
which is formulated for unsupervised learning use cases.
In unsupervised learning, no labels are provided, corresponding to the case of $\ntags = 0$.
In contrast, we make full use of existing labels to construct the topics,
which are represented by individual rows of $\Ymatrix$.
This formulation also provides an explicit representation for the \textit{semantic content} of each topic,
as $\Ytagmatrix$ contains weights indicating the relative importance of each tag.

Furthermore, this formulation allows us to construct explicit operational definitions of monosemy and polysemy.

\begin{definition}
Let %
$\datamatrix \approx \Xmatrix\transpose\Ymatrix
    = \Xmatrix\transpose
    \begin{pmatrix} \Ytagmatrix& \Yfeatmatrix
    \end{pmatrix}
    \label{eq:approx}
$ %
be a labelled topic model with $\Ytagmatrix \in \{0, 1\}^{\nrank \times \ntags}$ having only observed matrix elements.
Then, a tag $\sometag$ is \textit{$n$-semic} where $n$ is the corresponding column sum of $\Ytagmatrix$, i.e.\ $n = \sum_j \left(\Ytagmatrix\right)_{j,\sometag}$.
A \textit{monosemic} tag is 1-semic.
A \textit{polysemic} tag is $n$-semic with $n>1$.
\end{definition}

Unless otherwise noted, we consider only monosemic tags in this paper, which allows us to specialize further.

\begin{definition}
\label{def:1topicmodel}
Let $\datamatrix$ be defined as in \Cref{def:topicmodel}.
Then, a \textit{perfectly monosemic labelled topic model} is a labelled topic model such that
$\Ytagmatrix$ is the identity matrix.
An \textit{approximately monosemic labelled topic model} is a labelled topic model such that
$\Ytagmatrix$ is a non-negative, diagonally dominant matrix with unit diagonal.
\end{definition}

Our notion of perfect monosemy is similar to separability \citep{Donoho2003,Arora2012},
but specialized to having some tag $\sometag$ as the specific anchor term.

\subsection{An illustrative example}%
\label{sec:example}

Suppose we have the following data systems:

\begin{enumerate}

\item[A:]
A database with the semantic tag ``account'', containing one table, ACCOUNTS,
which has the columns ACCT\_ID, DATE, FIRST, LAST, and ACCT\_BAL, and ,

\item[B:]
A database with the semantic tag ``transaction'',
containing one table, TXNS,
which has the columns ACCT\_ID, DATE, PAYEE, AMOUNT, and ACCT\_BAL, and

\item[C:]
A legacy system with no tags, no table names, 
but containing the columns ACCT\_BAL, DATE, FIRST, LAST, AMOUNT, and PAYEE.
\end{enumerate}

\paragraph{Encoding the data matrix}
Treating each table as a separate data set, we can encode the combined semantic and physical metadata into the matrix $\datamatrix =$

\begin{tabular}{c|c|c|c|c|c|c|}
\multicolumn{1}{c}{}&
\multicolumn{1}{c}{{\footnotesize{}tag}} & \multicolumn{1}{c}{{\footnotesize{}tag}} & 
\multicolumn{1}{c}{{\footnotesize{}table}} & \multicolumn{1}{c}{{\footnotesize{}table}} & \multicolumn{1}{c}{{\footnotesize{}column}}
\tabularnewline
\multicolumn{1}{c}{} &
\multicolumn{1}{c}{{\footnotesize{}account}} & \multicolumn{1}{c}{{\footnotesize{}transaction}} &
\multicolumn{1}{c}{{\footnotesize{}ACCOUNTS}} & \multicolumn{1}{c}{{\footnotesize{}TXNS}} & \multicolumn{1}{c}{{\footnotesize{}ACCT\_BAL}} \tabularnewline
\cline{2-6}
{\footnotesize{}A} & 1  & \missing  & 1  & 0  & 1   \tabularnewline
\cline{2-6}
{\footnotesize{}B} & \missing  & 1  & 0  & 1  & 1   \tabularnewline
\cline{2-6}
{\footnotesize{}C} & \missing  & \missing  & 0 & 0 & 1 \tabularnewline
\cline{2-6}
\end{tabular}

\begin{tabular}{c|c|c|c|c|c|c|}
\multicolumn{1}{c}{}&
\multicolumn{1}{c}{{\footnotesize{}column}} & \multicolumn{1}{c}{{\footnotesize{}column}} & \multicolumn{1}{c}{{\footnotesize{}column}} & \multicolumn{1}{c}{{\footnotesize{}column}} & \multicolumn{1}{c}{{\footnotesize{}column}} & \multicolumn{1}{c}{{\footnotesize{}column}}
\tabularnewline
\multicolumn{1}{c}{} &
\multicolumn{1}{c}{{\footnotesize{}ACCT\_ID}} & \multicolumn{1}{c}{{\footnotesize{}AMOUNT}} & \multicolumn{1}{c}{{\footnotesize{}DATE}} & \multicolumn{1}{c}{{\footnotesize{}FIRST}} & \multicolumn{1}{c}{{\footnotesize{}LAST}} & \multicolumn{1}{c}{{\footnotesize{}PAYEE}}
\tabularnewline
\cline{2-7}
{\footnotesize{}A} & 1  & 0 & 1  & 1  & 1  & 0 \tabularnewline
\cline{2-7}
{\footnotesize{}B} & 1  & 1 & 1  & 0  & 0  & 1 \tabularnewline
\cline{2-7}
{\footnotesize{}C} & 0 & 1 & 1 & 1 & 1 & 1\tabularnewline
\cline{2-7}
\end{tabular}

\medskip

An important note is to distinguish between how absent tags and absent features are encoded.
An feature that does not exist in a data set is truly absent and is encoded with a 0.
In contrast, a tag that does not exist may actually be missing, inadvertently or erroneously, and is encoded as a missing value.
This asymmetry in the meaning of absent tags and absent features motivates the definition of the data matrix in \Cref{def:topicmodel}.

The goal is to factorize $\datamatrix$ into a perfectly monosemic labelled topic model (\Cref{def:1topicmodel}).



\subsection{Latent semantic indexing (LSI)}
\label{sec:lsi}

We now show that latent semantic indexing (LSI) \citep{Deerwester1990,Berry1999}
yields a labelled topic model (\Cref{def:topicmodel}).

\begin{definition}
The \textit{$k$-truncated singular value decomposition} ($k$-SVD, truncated SVD) of a matrix $\datamatrix$ is the matrix factorization
\begin{equation}
\begin{subequations}
\datamatrix \approx \datamatrix_\nrank = U S V\transpose
\end{subequations}
\label{eq:tsvd}
\end{equation}
where $U$ is a $\ndatasets\times\nrank$ orthogonal matrix,
$V$ is a $(\ntags+\nfeatures)\times\nrank$ orthogonal matrix,
$S$ is a diagonal $\nrank\times\nrank$ orthogonal matrix,
and $\sqrt S$ is the matrix square root of $S$ (in this case,
diagonal $\nrank\times\nrank$ orthogonal matrix with diagonal entries
$\sqrt S_{i,i})$, ordered such that $S_{1,1} \ge S_{2,2} \ge \cdots \ge S_{k,k} \ge 0$.
\end{definition}

LSI is computed from the truncated SVD of \eqref{eq:tsvd} by:

\begin{align}
\begin{split}
\Xmatrix_{\mbox{\textsc{lsi}}} & =  \sqrt{S} U\transpose, \\
\Ymatrix_{\mbox{\textsc{lsi}}} & =  \sqrt{S} V\transpose.
\label{eq:lsi}
\end{split}
\end{align}

The ordinary use case for LSI is simply the case where $\ntags=0$,
i.e.\ when no semantic tags are provided. 

In our example, take $k=2$ and thence obtain from LSI:

$X=$%
\begin{tabular}{c|c|c|c|}
\multicolumn{1}{c}{} & \multicolumn{1}{c}{{\footnotesize{}A}} & \multicolumn{1}{c}{{\footnotesize{}B}} & \multicolumn{1}{c}{{\footnotesize{}C}}\tabularnewline
\cline{2-4}
\textit{\footnotesize{}account} & -0.95 & -0.55 & 0.78 \tabularnewline
\cline{2-4} 
\textit{\footnotesize{}transaction} & 0.81 & -0.47 & 0.66 \tabularnewline
\cline{2-4}
\end{tabular}

$Y_{t}=$%
\begin{tabular}{c|c|c|}
\multicolumn{1}{c}{} & \multicolumn{1}{c}{{\footnotesize{}tag}} & \multicolumn{1}{c}{{\footnotesize{}tag}}\tabularnewline
\multicolumn{1}{c}{} & \multicolumn{1}{c}{{\footnotesize{}account}} & \multicolumn{1}{c}{{\footnotesize{}transaction}}\tabularnewline
\cline{2-3} 
\textit{\footnotesize{}account} & -0.26 & 0.26\tabularnewline
\cline{2-3}
\textit{\footnotesize{}transaction} & -0.18 & 0.18\tabularnewline
\cline{2-3}
\end{tabular}

$Y_{f}=$%
\begin{tabular}{c|c|c|c|c|}
\multicolumn{1}{c}{} & \multicolumn{1}{c}{{\footnotesize{}table}} & \multicolumn{1}{c}{{\footnotesize{}table}} & \multicolumn{1}{c}{{\footnotesize{}column}} & \multicolumn{1}{c}{{\footnotesize{}column}}
\tabularnewline
\multicolumn{1}{c}{} & \multicolumn{1}{c}{{\footnotesize{}ACCOUNTS}} & \multicolumn{1}{c}{{\footnotesize{}TXNS}} & \multicolumn{1}{c}{{\footnotesize{}ACCT\_BAL}} & \multicolumn{1}{c}{{\footnotesize{}ACCT\_ID}}
\tabularnewline
\cline{2-5}
\textit{\footnotesize{}account} & -0.53 &  0.53 & 0.00 &  0.00\tabularnewline
\cline{2-5}
\textit{\footnotesize{}transaction} & -0.35 & -0.35 & 0.00 & -0.71\tabularnewline
\cline{2-5}
\end{tabular}

$\textcolor{white}{Y_{f}=}$%
\begin{tabular}{c|c|c|c|c|c|}
\multicolumn{1}{c}{} & \multicolumn{1}{c}{{\footnotesize{}column}} & \multicolumn{1}{c}{{\footnotesize{}column}} & \multicolumn{1}{c}{{\footnotesize{}column}} & \multicolumn{1}{c}{{\footnotesize{}column}} & \multicolumn{1}{c}{{\footnotesize{}column}}\tabularnewline
\multicolumn{1}{c}{} & \multicolumn{1}{c}{{\footnotesize{}AMOUNT}} & \multicolumn{1}{c}{{\footnotesize{}DATE}} & \multicolumn{1}{c}{{\footnotesize{}FIRST}} & \multicolumn{1}{c}{{\footnotesize{}LAST}} & \multicolumn{1}{c}{{\footnotesize{}PAYEE}}\tabularnewline
\cline{2-6}
\textit{\footnotesize{}account} & 0.53 &  0.0 & -0.53 & -0.53 & 0.53\tabularnewline
\cline{2-6}
\textit{\footnotesize{}transaction} &  0.35 & 0.0 &  0.35 &  0.35 & 0.35\tabularnewline
\cline{2-6}
\end{tabular}

\medskip

As can be seen, the topics of $\Ymatrix$ are difficult to interpret, because each row of $\Yfeatmatrix$ is not purely
about a single label.
It appears that the first topic is effectively the combination $t_1 = $``account'' - ``transaction'' in equal strengths,
while the second topic is $t_2 = $ ``account'' + ``transaction'' in equal strengths.
It's not clear what the subtraction in the first topic means.
The observation that negative coefficients are difficult to interpret was the motivation behind non-negative matrix factorization \citep{Lawton1971,Lee1999}
and subsequent work to build topic models from such factorizations \citep{Arora2012}.

One may attempt to build a more interpretable topic model from the result LSI,
by noting that $s_1 = t_1 + t_2$ and $s_2 = t_2 - t_1$ are two new topics that collectively
contain with same information content as the original topics $t_1$ and $t_2$,
while furthermore being ``purer'' in the sense that $s_1$ is only about ``account'' while $s_2$ is only about ``transaction''.
We formalize this observation in the next section.

\subsection{LSI with generalized gauge (GG-LSI)}
\label{sec:gglsi}

We now introduce a generalization of LSI that yields a perfectly monosemic labelled topic model in the sense of \Cref{def:1topicmodel}.
This generalization is possible because \eqref{eq:lsi} is not the only way to construct some product $\Xmatrix\transpose\Ymatrix$ from the truncated SVD.
In fact, any transformation $\somematrix: \Ymatrix \mapsto \tilde\Ymatrix = \somematrix\Ymatrix$ that forms new linear combinations of rows of the topic model in $\Ymatrix$ yields a new topic model $\tilde\Ymatrix$.
If the transformation matrix $\somematrix$ is nonsingular, then no information is lost in the construction of $\tilde\Ymatrix$ (up to floating-point rounding),
even though the rows of $\tilde\Ymatrix$ are no longer orthogonal.

\begin{lemma}
\label{thm:gauge}
Let $\somematrix$ be an invertible matrix and
$\datamatrix = \Xmatrix^{\transpose}\Ymatrix$ be some matrix factorization of $\datamatrix$.
Then, $\tilde\Xmatrix = \left(\somematrix^{-1}\right)^{\transpose} \Xmatrix$,
$\tilde\Ymatrix = \somematrix \Ymatrix$ is also a matrix factorization of $\datamatrix$, i.e.\ 
${\tilde\Xmatrix}^{\transpose}\tilde\Ymatrix = \datamatrix$.
\end{lemma}

The proof follows by direct calculation.

This lemma shows that $\somematrix$ plays the role of a gauge in the sense used in theoretical physics \citep{Thirring1997},
in that the gauge transformation $(\Xmatrix, \Ymatrix) \rightarrow (\tilde\Xmatrix, \tilde\Ymatrix)$, 
leaves the matrix factorization $\datamatrix = \Xmatrix^{\transpose}\Ymatrix$ unchanged.
The parameters of the gauge transformation are simply the matrix elements of $\somematrix$.

\begin{theorem}
\label{thm:lsi-gauge}
Let $\datamatrix \approx \tilde\datamatrix = U S V^{\transpose}$ be the truncated SVD of $\datamatrix$.
Assume that all entries along the diagonal of $S$ are strictly positive.
Furthermore, set $\Xmatrix= \sqrt S U^{\transpose}$ and $\Ymatrix= \sqrt S V^{\transpose}$ as in LSI, with 
partition $V = \begin{pmatrix}V_\atag & V_\afeat\end{pmatrix}$.
Also set $\somematrix = V_\atag S^{-1/2}$ in \Cref{thm:gauge}.
Then, ${\tilde\Xmatrix}^{\transpose}\tilde\Ymatrix = \tilde\datamatrix$ is a valid matrix factorization of $\tilde\datamatrix$, and furthermore $\somematrix\Ymatrix = \begin{pmatrix}I & V_\atag V_\afeat^{\transpose}\end{pmatrix}$. In other words, ${\tilde\Xmatrix}^{\transpose}\tilde\Ymatrix$ defines a perfectly monosemic labelled topic model.
\end{theorem}

The proof follows by direct calculation, noting that $V$ is an orthogonal matrix, $S$ is invertible, and $V_\atag$ is of full rank.


For the example of \Cref{sec:example}, GG-LSI yields the topic model:

$X=$%
\begin{tabular}{c|c|c|c|}
\multicolumn{1}{c}{} & \multicolumn{1}{c}{{\footnotesize{}A}} & \multicolumn{1}{c}{{\footnotesize{}B}} & \multicolumn{1}{c}{{\footnotesize{}C}}\tabularnewline
\cline{2-4}
\textit{\footnotesize{}account} & 0.13 & -0.37 & 0.25 \tabularnewline
\cline{2-4}
\textit{\footnotesize{}transaction} & 0.25 & -0.09 & -0.17 \tabularnewline
\cline{2-4}
\end{tabular}

$Y_{t}=$%
\begin{tabular}{c|c|c|}
\multicolumn{1}{c}{} & \multicolumn{1}{c}{{\footnotesize{}tag}} & \multicolumn{1}{c}{{\footnotesize{}tag}}\tabularnewline
\multicolumn{1}{c}{} & \multicolumn{1}{c}{{\footnotesize{}account}} & \multicolumn{1}{c}{{\footnotesize{}transaction}}\tabularnewline
\cline{2-3}
\textit{\footnotesize{}account} & 1.00 & 0.00\tabularnewline
\cline{2-3}
\textit{\footnotesize{}transaction} & 0.00 & 1.00\tabularnewline
\cline{2-3}
\end{tabular}

$Y_{f}=$%
\begin{tabular}{c|c|c|c|c|}
\multicolumn{1}{c}{} & \multicolumn{1}{c}{{\footnotesize{}table}} & \multicolumn{1}{c}{{\footnotesize{}table}} & \multicolumn{1}{c}{{\footnotesize{}column}} & \multicolumn{1}{c}{{\footnotesize{}column}}
\tabularnewline
\multicolumn{1}{c}{} & \multicolumn{1}{c}{{\footnotesize{}ACCOUNTS}} & \multicolumn{1}{c}{{\footnotesize{}TXNS}} & \multicolumn{1}{c}{{\footnotesize{}ACCT\_BAL}} & \multicolumn{1}{c}{{\footnotesize{}ACCT\_ID}}
\tabularnewline
\cline{2-5}
\textit{\footnotesize{}account} & 0.62 & 1.11 & 0.00 & 1.53\tabularnewline
\cline{2-5}
\textit{\footnotesize{}transaction} & 0.38 & 1.69 & 0.00 & 2.13 \tabularnewline
\cline{2-5}
\end{tabular}

$\textcolor{white}{Y_{f}=}$%
\begin{tabular}{c|c|c|c|c|c|}
\multicolumn{1}{c}{} & \multicolumn{1}{c}{{\footnotesize{}column}} & \multicolumn{1}{c}{{\footnotesize{}column}} & \multicolumn{1}{c}{{\footnotesize{}column}} & \multicolumn{1}{c}{{\footnotesize{}column}} & \multicolumn{1}{c}{{\footnotesize{}column}}\tabularnewline
\multicolumn{1}{c}{} & \multicolumn{1}{c}{{\footnotesize{}AMOUNT}} & \multicolumn{1}{c}{{\footnotesize{}DATE}} & \multicolumn{1}{c}{{\footnotesize{}FIRST}} & \multicolumn{1}{c}{{\footnotesize{}LAST}} & \multicolumn{1}{c}{{\footnotesize{}PAYEE}}\tabularnewline
\cline{2-6}
\textit{\footnotesize{}account} & -0.03 & 0.00 & -0.48 & -0.48 & -0.03\tabularnewline
\cline{2-6}
\textit{\footnotesize{}transaction} & -0.55 & 0.00 & -1.87 & -1.87 & -0.55\tabularnewline
\cline{2-6}
\end{tabular}

\medskip

\subsection{Generalized low-rank models}
\label{sec:glrm}

We now consider a matrix factorization approach based on the generalized low-rank model (GLRM) formalism \citep{Udell2016}.
These matrices are determined by minimizing the loss function
\begin{equation}
    \loss(\Xmatrix, \Ymatrix|A) = \sum_{(i,j)\in\Omega} L_{i,j}(\Xmatrix_{i,:} \Ymatrix_{j,:}, \datamatrix_{i,j})
    + \sum_{i=1}^{\ndatasets} r_i (\Xmatrix_{i,:}) 
    + \sum_{j=1}^{\nfeatures+\ntags} \tilde{r}_j (\Ymatrix_{j,:}),
\label{eq:glrm}
\end{equation}
where $\Omega$ is the observed set of \Cref{def:obs},
$\Xmatrix_{i,:}$ is the $i$th row of $\Xmatrix$,
$\Ymatrix_{i,:}$ is the $j$th row of $\Ymatrix$,
$L_{i,j}$ is the elementwise loss function for the matrix element $\datamatrix_{i,j}$,
$r_i$ is a regularizer the $i$th row of $\Xmatrix$,
and 
$\tilde{r}_j$ is the regularizer for the $j$th row of $\Ymatrix$.
The full generality of this formalism is described in detail in \citep{Udell2016}.

Unless otherwise stated, we specialize to the choice of logistic loss\footnote{Our definition varies from that in \citet{Udell2016}, as they consider non-missing entries of $\datamatrix$ to be over $\{\pm1\}$ and not $\{0,1\}$ as is done here.} $L(u, a) = \log(1 + \exp(-(2a-1)u))$
and $l_1$-regularizers $r(z; \regstrength) = \regstrength \sum_j \left\vert z_j \right\vert_1$.

The entry-wise nature of the GLRM formalism also admits a straightforward specialization to the problem of learning an approximately monosemic labelled topic model as in \Cref{def:1topicmodel},
simply by adding to \eqref{eq:glrm} generalized regularizers of the form
\begin{equation}
\begin{subequations}
    g_\sometag(y) = \begin{cases}
    0 & \mbox{if } \sum_{\sigma \ne \sometag}
             |y_\sigma| \ge 1, \\
    \infty & \mbox{else}
    \end{cases},\quad
    h_\sometag(y) = \begin{cases}
    0 & \mbox{if } y_\sometag = 1, \\
    \infty & \mbox{else}
    \end{cases}.
\end{subequations}
        \label{eq:penalty}
\end{equation}

Missing entries are not part of the observed set $\Omega$ and are excluded from the computation of the loss function $\loss$ of \eqref{eq:glrm} during training.
Thus, the model is not penalized with regard to missing tags, but is penalized for predicting features are present when they are actually absent.
The model is penalized for mispredicting the absence of both tags and features, when they are actually present.

\Cref{alg:glrm} describes a complete algorithm for the topic modeling approach we have presented.
A $k$-SVD of the imputed data matrix computed,
which is used to construct the GG-LSI factorization of \Cref{sec:gglsi},
which is in turn used as a warm start for the GLRM fitting procedure of \Cref{sec:glrm}
with randomized subsampling.

In the experiments below, we use a very simple imputation with a completely uninformative prior of
replacing each missing topic label with an entry of 0.5.
Alternative imputation methods are of course possible;
one noteworthy variant of potential future interest is the Lanczos expectation--maximization algorithm
for self-imputing $k$-SVD \citep{Kurucz2007}.
We do not consider alternative imputations further in this paper.
We also found a small amount of noise $\epsilon$ to be helpful for the GLRM training to avoid getting trapped prematurely in local minima.
Unless otherwise specified, we use noise sampled from the uniform distribution $D = [0, 0.001]$.

\begin{algorithm}
\SetAlgoLined
\KwData{Data matrix $\datamatrix$, noise distribution $D$,
regularization strength $\lambda$.}
\KwResult{The matrix factorization $(\Xmatrix, \Ymatrix)$.}
Impute missing values in $\datamatrix$.\\
Compute an approximate, low-accuracy $k$-SVD $\datamatrix \approx U S V^{\transpose}$ as in \eqref{eq:tsvd}.\\
Calculate the GG-LSI of \Cref{thm:lsi-gauge} with
$\datamatrix \approx \Xmatrix_0^{\transpose} \Ymatrix_0$,
$\Xmatrix_0 = V_\atag S U^{\transpose} + \epsilon$,
$\Ymatrix_0 = V_\atag V^{\transpose}$,
with some added elementwise noise $\epsilon\sim D$.\\
Select a randomized subsample $\tilde\Omega \subset \Omega$.\\
Fit the GLRM with loss function \eqref{eq:glrm} over the subset $\tilde\Omega$ with the penalty functions \eqref{eq:penalty} and regularization strength $\lambda$,
initialized with $(\Xmatrix, \Ymatrix) = (\Xmatrix_0, \Ymatrix_0)$ as defined above.
\caption{Topic modeling algorithm}
\label{alg:glrm}
\end{algorithm}

\subsection{Revisiting the simple example}
\label{sec:example2}

We now revisit the example introduced in \Cref{sec:example} with a detailed post-training analysis.
Using the GLRM formalism described above, we get a final loss of $\loss = 2.26$ (rounded to 2 decimals)
and the matrices:

$X=$%
\begin{tabular}{c|c|c|c|}
\multicolumn{1}{c}{} & \multicolumn{1}{c}{{\footnotesize{}A}} & \multicolumn{1}{c}{{\footnotesize{}B}} & \multicolumn{1}{c}{{\footnotesize{}C}}\tabularnewline
\cline{2-4} 
\textit{\footnotesize{}account} & {-2.66} & {0.55} & {2.24}\tabularnewline
\cline{2-4}
\textit{\footnotesize{}transaction} & {-1.87} & {3.22} & {1.00}\tabularnewline
\cline{2-4}
\end{tabular}

$Y_{t}=$%
\begin{tabular}{c|c|c|}
\multicolumn{1}{c}{} & \multicolumn{1}{c}{{\footnotesize{}tag}} & \multicolumn{1}{c}{{\footnotesize{}tag}}\tabularnewline
\multicolumn{1}{c}{} & \multicolumn{1}{c}{{\footnotesize{}account}} & \multicolumn{1}{c}{{\footnotesize{}transaction}}\tabularnewline
\cline{2-3} 
\textit{\footnotesize{}account} & 1.00 & 0.57\tabularnewline
\cline{2-3}
\textit{\footnotesize{}transaction} & -1.00 & 1.00\tabularnewline
\cline{2-3}
\end{tabular}

$Y_{f}=$%
\begin{tabular}{c|c|c|c|c|}
\multicolumn{1}{c}{} & \multicolumn{1}{c}{{\footnotesize{}table}} & \multicolumn{1}{c}{{\footnotesize{}table}} & \multicolumn{1}{c}{{\footnotesize{}column}} & \multicolumn{1}{c}{{\footnotesize{}column}}
\tabularnewline
\multicolumn{1}{c}{} & \multicolumn{1}{c}{{\footnotesize{}ACCOUNTS}} & \multicolumn{1}{c}{{\footnotesize{}TXNS}} & \multicolumn{1}{c}{{\footnotesize{}ACCT\_BAL}} & \multicolumn{1}{c}{{\footnotesize{}ACCT\_ID}}
\tabularnewline
\cline{2-5}
\textit{\footnotesize{}account} & -0.44 & -2.66 & 2.73 & 0.16\tabularnewline
\cline{2-5}
\textit{\footnotesize{}transaction} & -3.07 & 1.87 & 1.14 & 0.06 \tabularnewline
\cline{2-5}
\end{tabular}

$\textcolor{white}{Y_{f}=}$%
\begin{tabular}{c|c|c|c|c|c|}
\multicolumn{1}{c}{} & \multicolumn{1}{c}{{\footnotesize{}column}} & \multicolumn{1}{c}{{\footnotesize{}column}} & \multicolumn{1}{c}{{\footnotesize{}column}} & \multicolumn{1}{c}{{\footnotesize{}column}} & \multicolumn{1}{c}{{\footnotesize{}column}}\tabularnewline
\multicolumn{1}{c}{} & \multicolumn{1}{c}{{\footnotesize{}AMOUNT}} & \multicolumn{1}{c}{{\footnotesize{}DATE}} & \multicolumn{1}{c}{{\footnotesize{}FIRST}} & \multicolumn{1}{c}{{\footnotesize{}LAST}} & \multicolumn{1}{c}{{\footnotesize{}PAYEE}}\tabularnewline
\cline{2-6}
\textit{\footnotesize{}account} & 0.46 & 2.73 & 2.69 & 2.69 & 0.46\tabularnewline
\cline{2-6}
\textit{\footnotesize{}transaction} & 3.15 & 1.14 & -1.88 & -1.88 & 3.15\tabularnewline
\cline{2-6}
\end{tabular}

\medskip

\paragraph{Interpretation}

To understand the value of the loss function achieved, consider the properties of the logistic loss.
For large positive values of $x$, $L(0,x) \approx x$ and $L(1,-x) \approx x$.
The loss function therefore measures the disagreement between positive predictions and an observed false value, and negative predictions and an observed true value.
Each $x$ being evaluated is an odds ratio for a specific event.
Intuitively, $\loss$ is therefore the cumulative odds of having the observed data disagree with the model.

In this model of matrix factorization, the main interpretable result is contained in the $\Ymatrix$ matrix.
For example, $\Ytagmatrix(\mbox{\textit{transaction}}, \mbox{account}) = -1 < 0$
says that the presence of the ``account'' tag reduces the likelihood that the \textit{transaction} topic is present.
Similarly, the entries of the row 
$\Ytagmatrix(\mbox{\textit{account}}, :)$
describe how strongly the presence or absence is associated with the presence of the \textit{account} topic.
The features that most strongly predict the presence of the \textit{account} topic are (in descending order):

\begin{enumerate}
\item[1--2.]
The presence of the ACCT\_BAL and DATE columns,

\item[3--4.]
The presence of the FIRST and LAST columns, and

\item[5.]
The \textit{absence} of the TXNS table.

\end{enumerate}

\paragraph{Predicting missing tags}

Finally, we may use the results to predict missing tags, by calculating the probability that each tag should be assigned to each system.
To calculate the probabilities, construct the matrix $B = \Xmatrix'\Ymatrix_\atag$, whose entries contain the odds ratios for the tag to exist in that system.
Finally, apply the logistic function to each element,%
\footnote{This definition of the logistic function differs from the standard one of $\pi(x) = 1/(1-\exp(-x))$, since the Boolean domain we use is over $\{0, 1\}$ and not $\{\pm 1\}$.}
$\pi(B_{i,j}) = 1/(1 + \exp(-(2B_{i,j}-1))),$
%
which is the probability that the label $j$ is present in system $i$.
In this example, the matrix of probabilities we obtain is $\Pi_\atag = (\pi(B_{i,j}))_{i,j} =$

\begin{tabular}{c|c|c|}
\multicolumn{1}{c}{} & \multicolumn{1}{c}{{\footnotesize{}tag}} & \multicolumn{1}{c}{{\footnotesize{}tag}}\tabularnewline
\multicolumn{1}{c}{} & \multicolumn{1}{c}{{\footnotesize{}account}} & \multicolumn{1}{c}{{\footnotesize{}transaction}}\tabularnewline
\cline{2-3}
{\footnotesize{}A}  & \textbf{0.99} & 0.42\tabularnewline
\cline{2-3}
{\footnotesize{}B}  & 0.06 & \textbf{0.97}\tabularnewline
\cline{2-3}
{\footnotesize{}C}  & \textbf{0.77} & \textbf{0.91}\tabularnewline
\cline{2-3}
\end{tabular}

\medskip

The probabilities in bold are greater than 0.5, indicating greater than even chance that the tag should exist.
In this example, the two declared tags (A, account) and (B, transaction) have probabilities greater than 0.5, whereas the missing pairs (A, transaction) and (B, transaction) are predicted not to exist. The probability that A should have the ``transaction'' tag is close to even, reflecting how the model picked up that system A has some features associated with the \textit{transaction} topic.
For our legacy system C, our model predicts that it should have both ``transaction'' and ``account'' labels.

A similar transformation of the matrix $B_\afeat = \Xmatrix'\Ymatrix_\afeat$ gives us probabilities that a feature should exist.
In our example, all features declared present in $\datamatrix$ have probabilities greater than 0.5,
and all features declared absent in $\datamatrix$ have probabilities less than 0.5, with the exception of (C, ACCT\_ID), which has probability 0.60.
This result may indicate that the legacy system C could be retrofit to introduce the ACCT\_ID column, in order to be better compatible with the expected data layout for transaction data.









\section{Randomized subsampling}
\label{sec:subsampling}

As described in \citet{Udell2016}, good matrix factorization results can be often obtained even when only a small fraction of the matrix elements are observed.
This result suggests that the original problem of minimizing \eqref{eq:glrm} over the entire set of observed entries $\Omega$ can be replaced by a subsample $\Omega^\prime \subset \Omega$.

The key insight is that a matrix $\somematrix$ of size $\nx \times \ny$ and rank $\nrank$ has much fewer than $\nx\ny$ independently observable entries, and therefore can be reconstructed by knowing $N \ll \nx\ny$ matrix elements.
The choice of minimal measurement number $N$ has been the subject of intense study in recent years \citep{Recht2011,Davenport2016}.
But first, here is an elementary argument for why $N$ must be at least $\nrank (\nx + \ny - \nrank)$.
This observation is not new \citep{Candes2009}, but we provide here a proof using elementary linear algebra.

\begin{theorem}
\label{thm:lowrank}
Let $\somematrix \in \mathbb{R}^{\nx \times \ny}$ be a real $\nx \times \ny$ matrix of rank $\nrank$.
Then, $\somematrix$ is fully parameterized by $\nrank (\nx + \ny - \nrank - 1)$ real numbers and an additional $\nrank$ positive real numbers.
\end{theorem}

\begin{proof}
A matrix of rank $\nrank$ can be factorized exactly into the truncated SVD \eqref{eq:tsvd} with $\nrank$ positive singular values \cite[\S 2.5.5, pp. 72--73]{Golub2013}, and is fully described by the parameters of $U$, $S$ and $V$.
Any $\nrank\times\nrank$ positive diagonal matrix $S$ can be parameterized by $\nrank$ positive real numbers.
Next, the space of orthonormal $U$s can be characterized by considering the columnwise construction of $U$,
akin to the orthogonalization procedure using Householder reflections \cite[\S 5.1.2, p. 209]{Golub2013}.
The first column of the orthonormal matrix $U$ has $\nx-1$ degrees of freedom (one normalization constraint), whereas the second column has $\nx-2$ degrees of freedom (normalization and orthogonality with the first column).
By induction, the $p$th column of $U$ has $\nx-p$ degrees of freedom, since it has to be normalized and orthogonal to the first $p-1$ columns.
Therefore, $U$ is parameterized by  $\sum_{p=1}^\nrank (\nx - p) = \nrank (\nx - (\nrank+1)/2)$ real numbers.
Similarly, $V$ is uniquely described by $\nrank (\ny - (\nrank+1)/2)$ real numbers.
Since the parameters of $U$, $S$ and $V$ can be independently chosen with no further restrictions, the result follows from counting the total number of parameters across the three matrices.
\end{proof}

\citet{Vandereycken2013} further proves that the space of all rank-$\nrank$ matrices forms a smooth manifold.

\Cref{thm:lowrank} above provides a reasonable minimum for the number of observations $N$.
For the problem of \textit{weak recovery}, i.e.,
to reconstruct one particular matrix,
the results of \citep{Xu2018},
as extended to rectangular matrices \citep{Rong2019},
use algebraic geometry to demonstrate that that $N \ge \nrank(\nx+\ny-\nrank)$ is sufficient (but not necessary), so long as a suitable basis can be constructed for the measurement operator \citep{Conca2015}. Similar results hold probabilistically in \citet{Eldar2012,Riegler2015}.
However, the results above do not specify which matrix elements of $\somematrix$ to observe.
Nevertheless, there are results to show that a constant factor of this threshold suffices for uniform random sampling. Empirical results of \citep{Recht2010} show recovery when $N \approx 2 \nrank(\nx+\ny-\nrank)$.
\citet{Candes2011} show that randomized subsampling with some constant factor, $N \propto \max(\nx, \ny) \nrank$ suffices.
Similarly, \citet{Davenport2014} show that for binary data matrices, given a constant error threshold, the number of necessary samples is $N \propto \max(\nx + \ny) \nrank$.

The results above are proved in a very general setting of measurement, but for our purposes, it suffices to interpret the need for a suitable basis as choosing $N$ different coordinates of the data matrix $\datamatrix$.
We therefore choose the lowest bound for weak reconstruction \citep{Candes2011,Davenport2014}:
\begin{equation}
    N = C \nrank (\nx + \ny - \nrank),
    \label{eq:subsample}
\end{equation}
where the elements of $\Omega^\prime$ are sampled uniformly at random with replacement \footnote{Experiments using sampling without replacement demonstrated significantly worse results.} from the coordinates of the observed entries, $\Omega$, with $C$ an empirical constant to be estimated.

To justify this procedure of random subsampling,
we study how the fitting of a rank-$k=1$ GLRM changes as a function of the number of samples, $N$,
from the Kaggle data set described in the next section.
The element-wise nature of the loss function $\loss$ \eqref{eq:glrm} means that we would expect that the loss $\loss$ should scale linearly with $N$, both in terms of magnitude and in computational cost, since \eqref{eq:glrm} is a direct summation over observed entries.
Empirical timings on a single machine suggest that computing the objective function over all observed samples in the matrix would take approximately 12 days, which is clearly impractical without significant parallelization.
The converged final values of the loss per sample, $\loss/N$, in \Cref{fig:subsampling_finalobj}, show more complicated behavior.
For $N \le 3$ there is a phase of catastrophic failure to fit, where the loss does not decrease very much from the initial value.
For $5 \le N \le \ceil{10^{4.25}} = 17783$ there seems to be a region of overfitting to the specific points being sampled.
There is then a transitional zone which stabilizes at $N \ge \ceil{10^{5.5}} = 316228$, above the choice $N=185270$ resulting from \eqref{eq:subsample} for rank $r=1$ with $C=1$, but below that choice for $C=2$. We therefore take $C=2$ to be a safe estimate for the necessary constant in \eqref{eq:subsample} to use randomized subsampling.
Our results seem to confirm the phase transition behavior reported in \citep{Recht2010}, albeit with three different phases, not two.

\begin{figure}
    \centering
    \includegraphics[keepaspectratio,width=0.7\columnwidth]{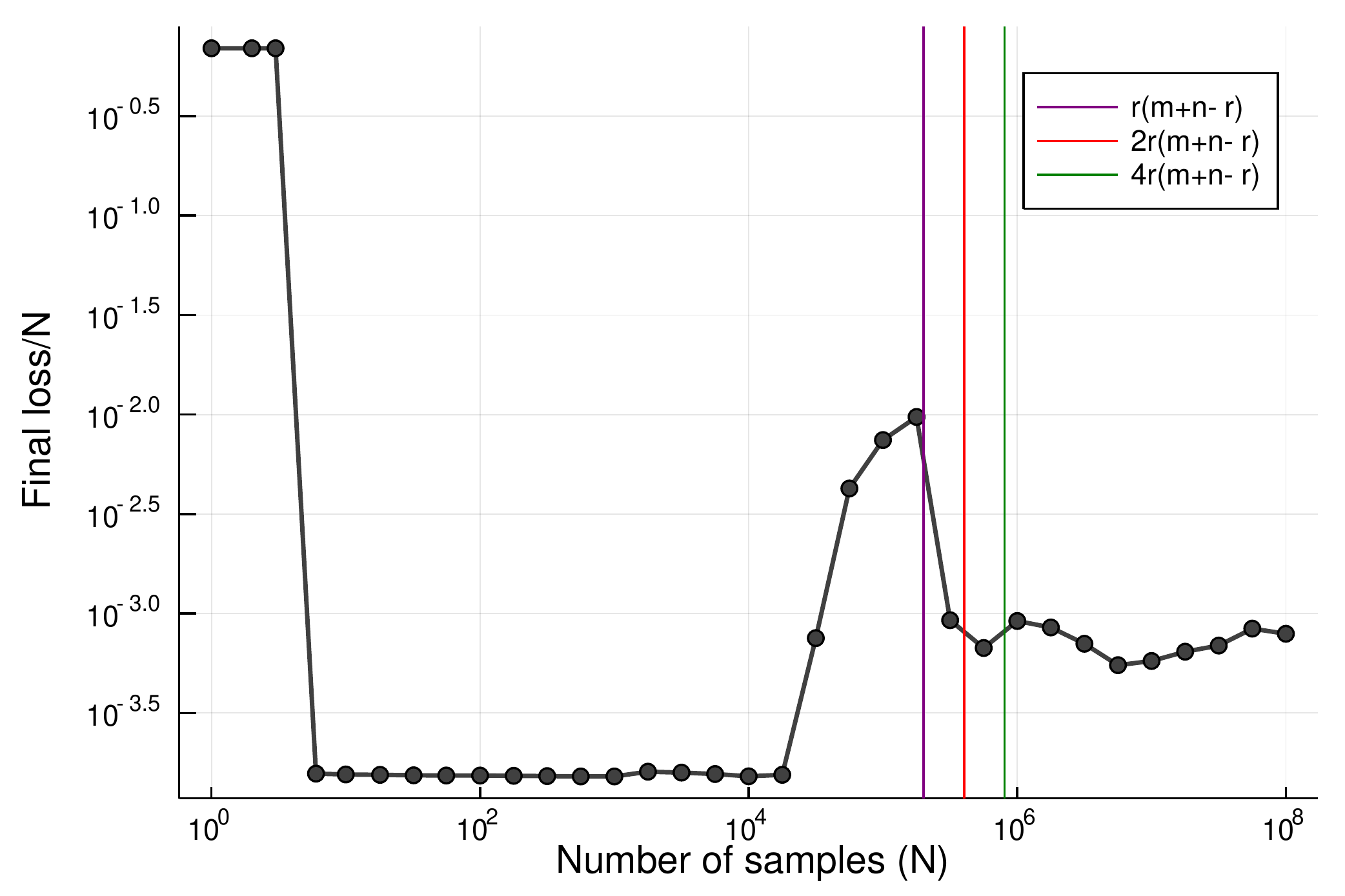}
    
    \caption{The final loss $\loss/N$ attained by sparse proximal gradient descent for a rank-1 generalized low-rank model with no regularization, as a function of the number of samples, $N = \sizeof{\Omega^\prime}$.}
    \label{fig:subsampling_finalobj}
\end{figure}

\begin{figure}
    \centering
    \includegraphics[keepaspectratio,width=0.7\columnwidth]{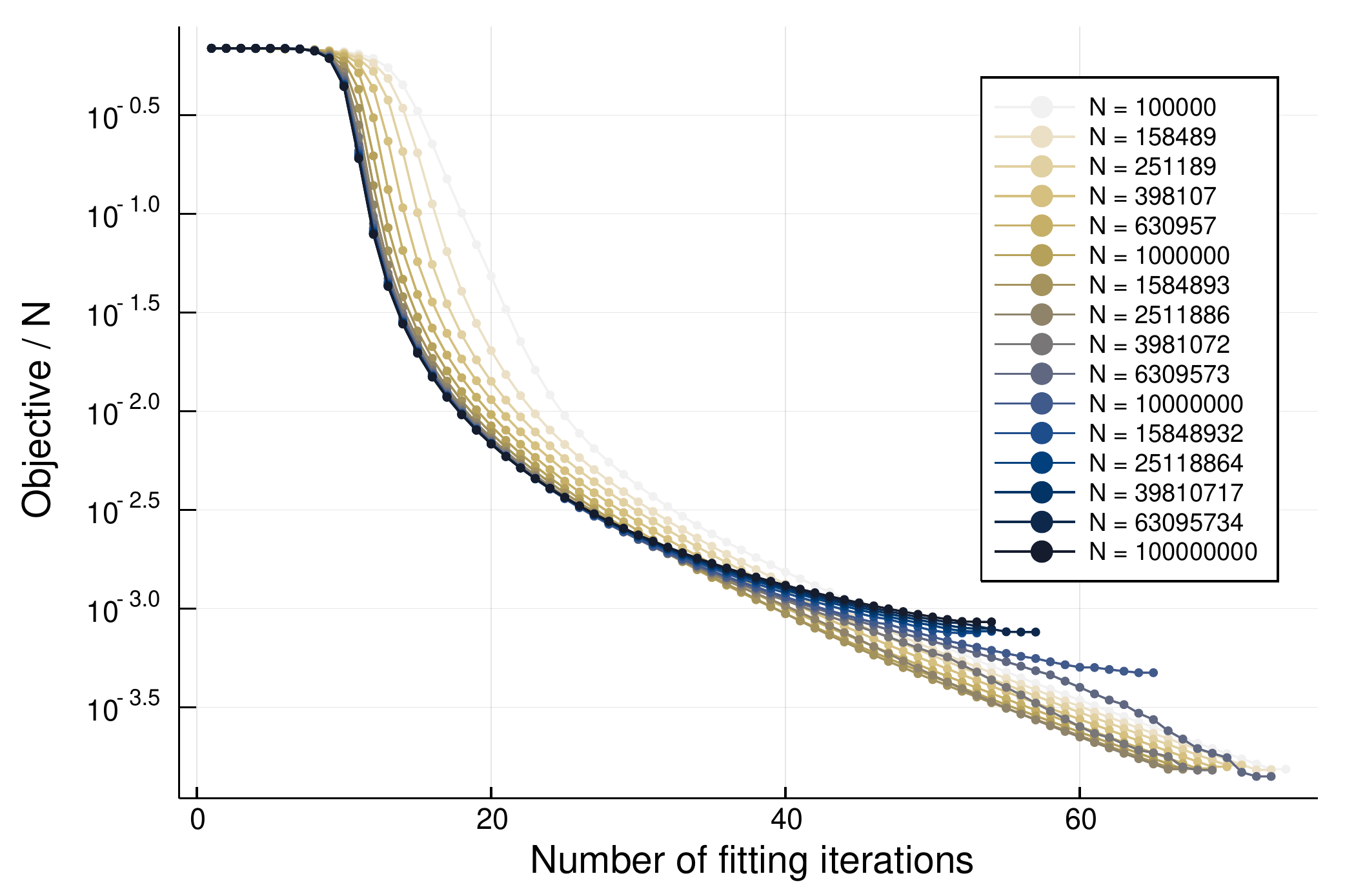}
    
    \caption{Convergence history of the loss $\loss/N$ attained by sparse proximal gradient descent for a rank-10 generalized low-rank model with no regularization, as a function of the number of samples, $N = \sizeof{\Omega^\prime}$.}
    \label{fig:subsampling_fit}
\end{figure}

    

\section{Semantic analysis of data sets on Kaggle}
\label{sec:kaggle}

In this section, we describe a data set curated from the metadata of data sets from Kaggle (\url{kaggle.com}).
The data set is comprised of metadata from 25,161 datasets that were accessible via the Kaggle Python API (v1.5.6) \citep{kaggleapi} as of December 4, 2019.
For each dataset, we collected 857 unique semantic tags, 44,399 unique file names, and 155,711 column names for each file that comprise the data set.
We collected these metadata programmatically through the Kaggle API and preprocessed the data to strip invalid Unicode characters.

Some representative examples are listed in \Cref{tab:my_label}. We see examples of data quality issues that are typical of collaborative tagging use cases as well as enterprise data catalogues, such as
datasets without subject tags,
file names without column names, and
column names without file names.
%
This data set therefore is a suitable candidate for the topic modeling approach we have described in \Cref{sec:glrm}.


\begin{table}
    \centering
    {\tiny
    \begin{tabular}{|c|c|c|c|}
    \hline 
    Dataset & tags & filenames & columns\tabularnewline
    \hline 
    4quant/depth-generation & image data & \multirow{3}{*}{depth\_training\_data.npz } & \multirow{3}{*}{-}\tabularnewline
    -lightfield-imaging & leisure &  & \tabularnewline
     & neural networks &  & \tabularnewline
    \hline 
    \multirow{4}{*}{colamas/index1min20190828} & \multirow{4}{*}{-} & CA.csv & Close\tabularnewline
     &  & HI.csv & High\tabularnewline
     &  & NQ.csv & Volume\tabularnewline
     &  &  & ... (8 total)\tabularnewline
    \hline 
    \multirow{4}{*}{dbahri/wine-ratings} & \multirow{4}{*}{-} & test.csv & Portugal\tabularnewline
     &  & train.csv & complex\tabularnewline
     &  & validation.csv & unoaked\tabularnewline
     &  &  & ... (48 total)\tabularnewline
    \hline 
    \multirow{4}{*}{insiyeah/musicfeatures} &  &  & chrome\_stft\tabularnewline
     & music & data.csv & mfcc17\tabularnewline
     & musical groups & data\_2genre.csv & spectral\_centroid\tabularnewline
     &  &  & ... (25 total) \tabularnewline
    \hline 
    \multirow{4}{*}{jackywang529/asdflkjaslj12dsfa} & \multirow{4}{*}{-} & \multirow{4}{*}{-} & 1\tabularnewline
     &  &  & 2\tabularnewline
     &  &  & amanda\tabularnewline
     &  &  & apple (4 total)\tabularnewline
    \hline 
    mrugankray/license-plates & law & - & -\tabularnewline
    \hline 
    nltkdata/wordnet & - & - & -\tabularnewline
    \hline 
    okc35fjh/pearson & education & pearson.xlsx & -\tabularnewline
    \hline 
    \end{tabular}
    }
    \caption{Some representative datasets and their metadata from \url{Kaggle.com} as collected using version 1.5.6 of the Kaggle API \citep{kaggleapi}.}
    \label{tab:my_label}
\end{table}


\subsection{Exploratory data analysis}

\Cref{fig:eda} shows the top 10 tags and features in our data set, as well as the cumulative frequencies.
It is interesting to note that the tags and features have quite different distributions.
While the rank ordering of features shows a power law tail distribution $r^{-\alpha_f}$ in the rank $r$ with exponent $\alpha_f = 2.07 \pm 0.01$,
the tags are much more heavy-tailed, with a corresponding exponent of $\alpha_t = 0.853 \pm 0.001$.
The presence of many rare tags poses a challenge for the standard approach of training an independent classifier
for each tag:
the classifiers for rare tags do not have sufficient signal in the training data,
resulting in poor performance overall.




\begin{figure}

\centering
\includegraphics[keepaspectratio=true,width=0.45\linewidth,height=\textheight]{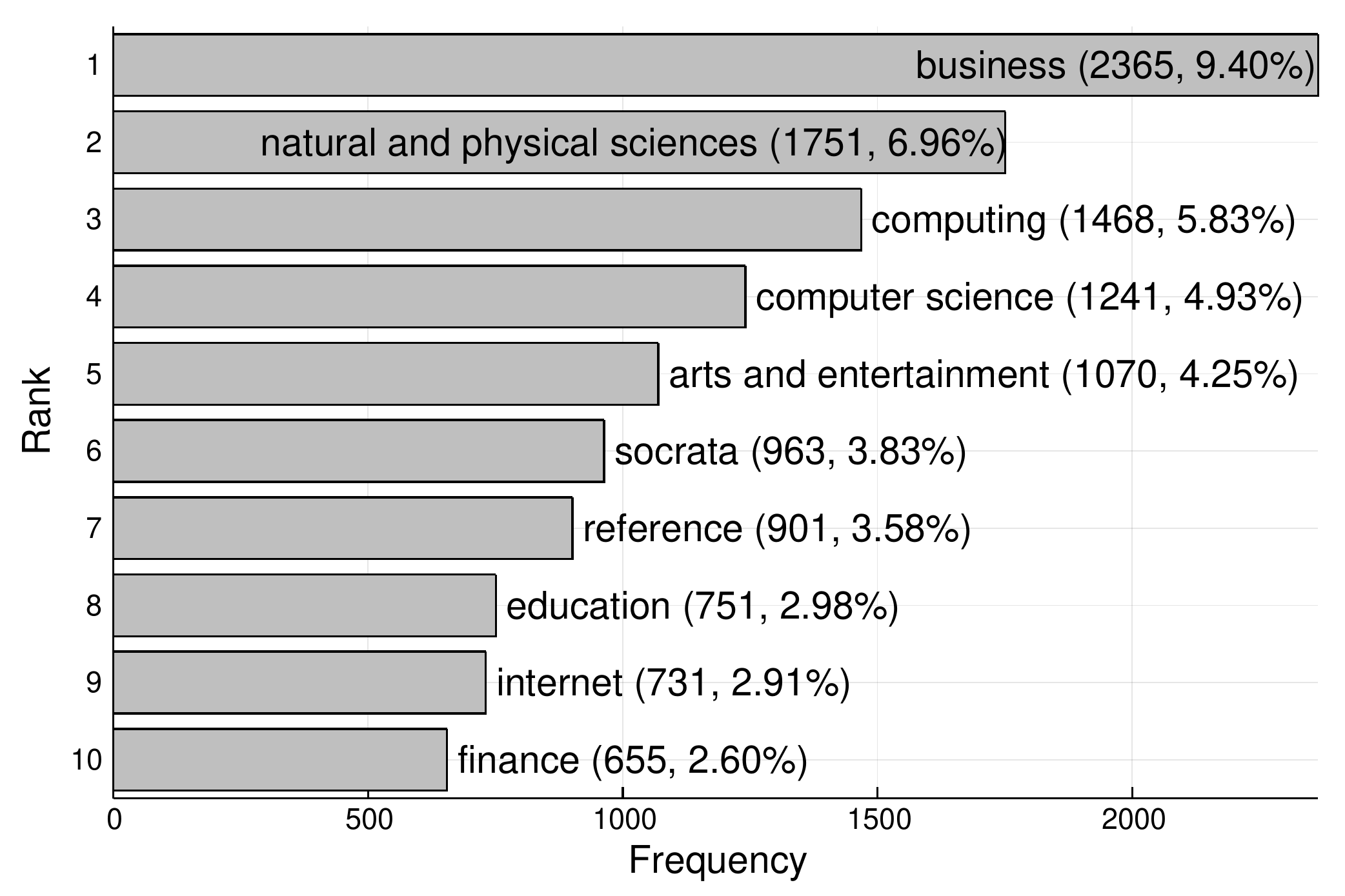}%
\includegraphics[keepaspectratio=true,width=0.45\linewidth,height=\textheight]{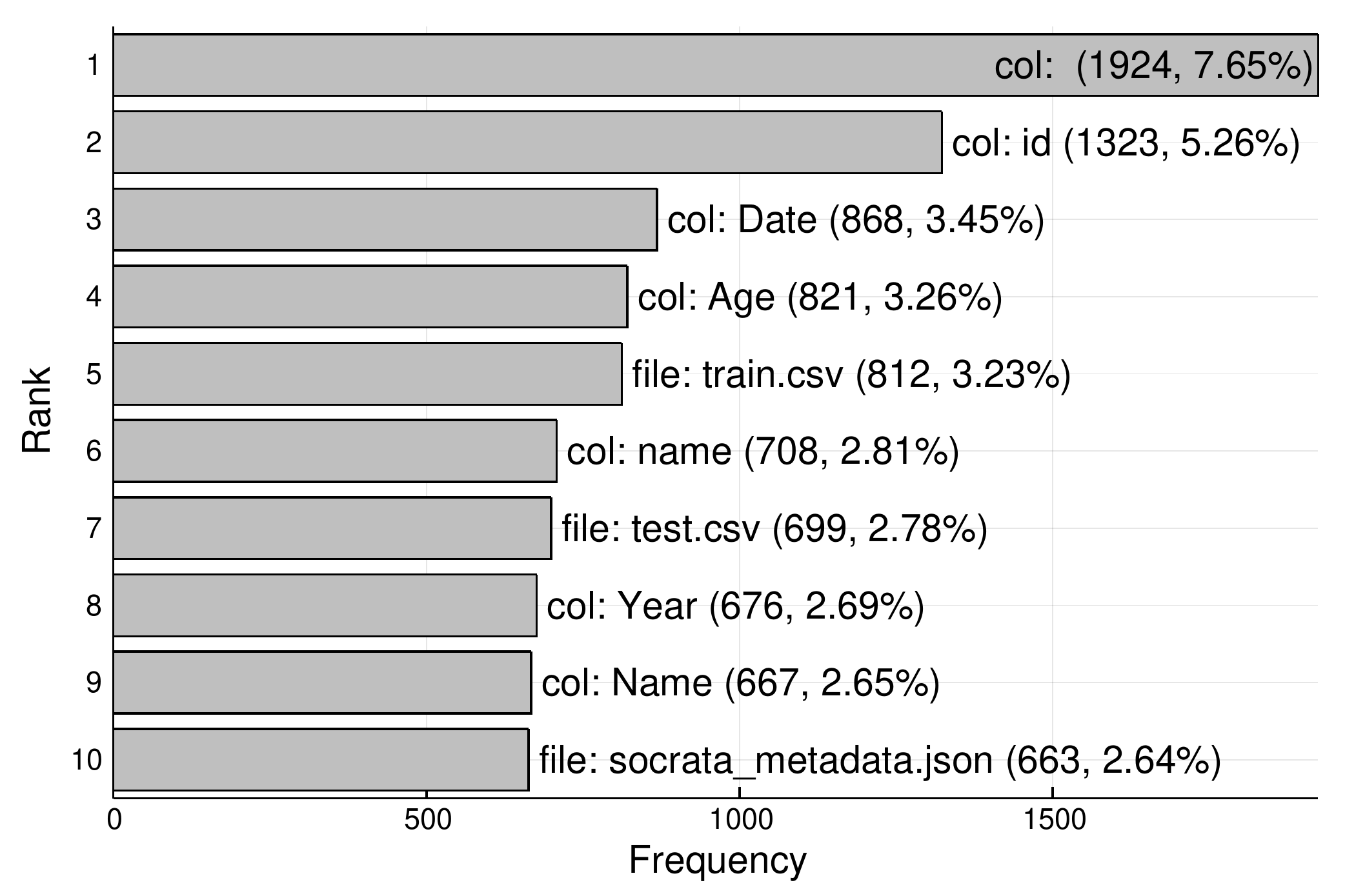}

\includegraphics[keepaspectratio=true,width=0.45\linewidth,height=\textheight]{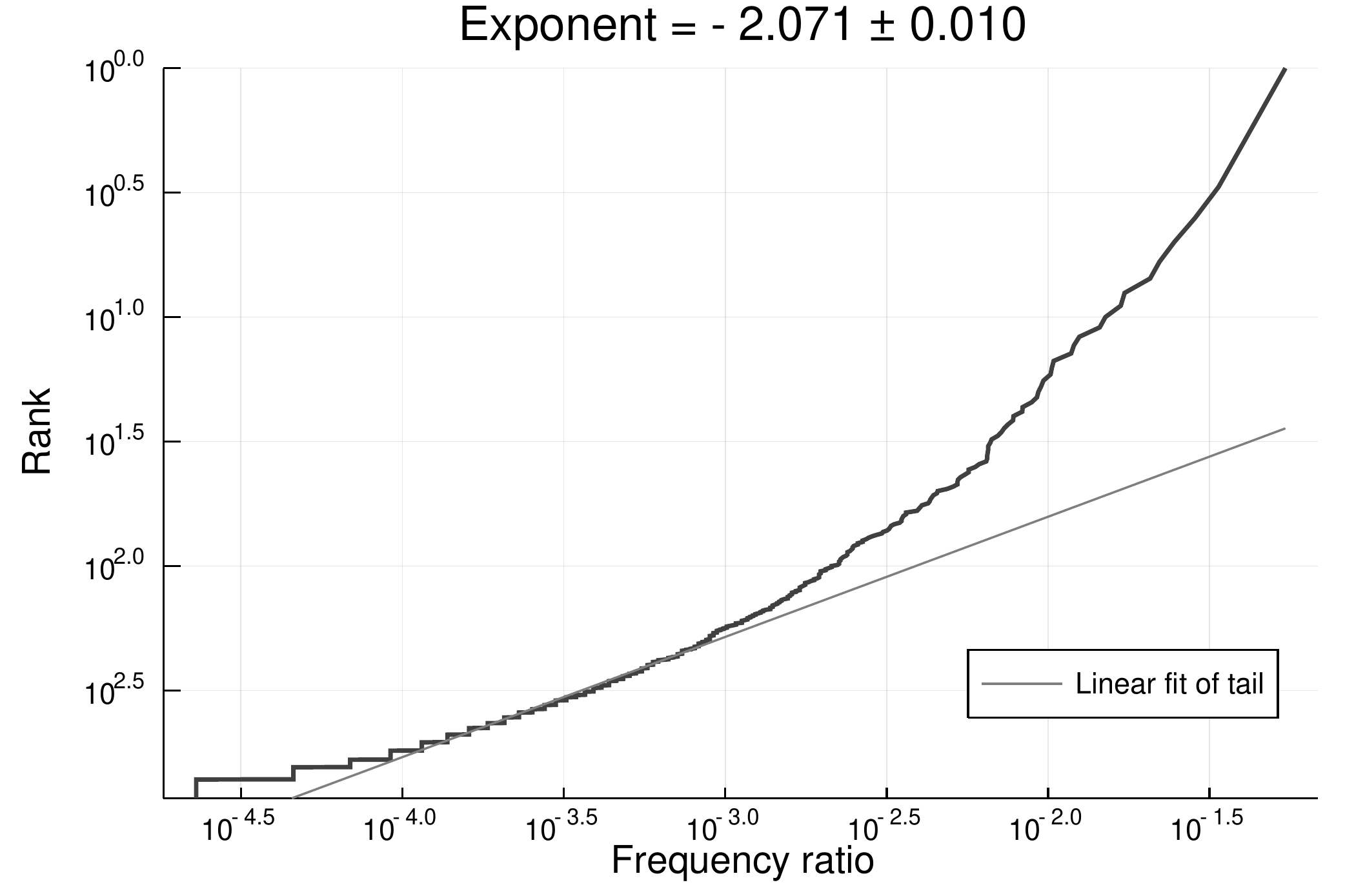}%
\includegraphics[keepaspectratio=true,width=0.45\linewidth,height=\textheight]{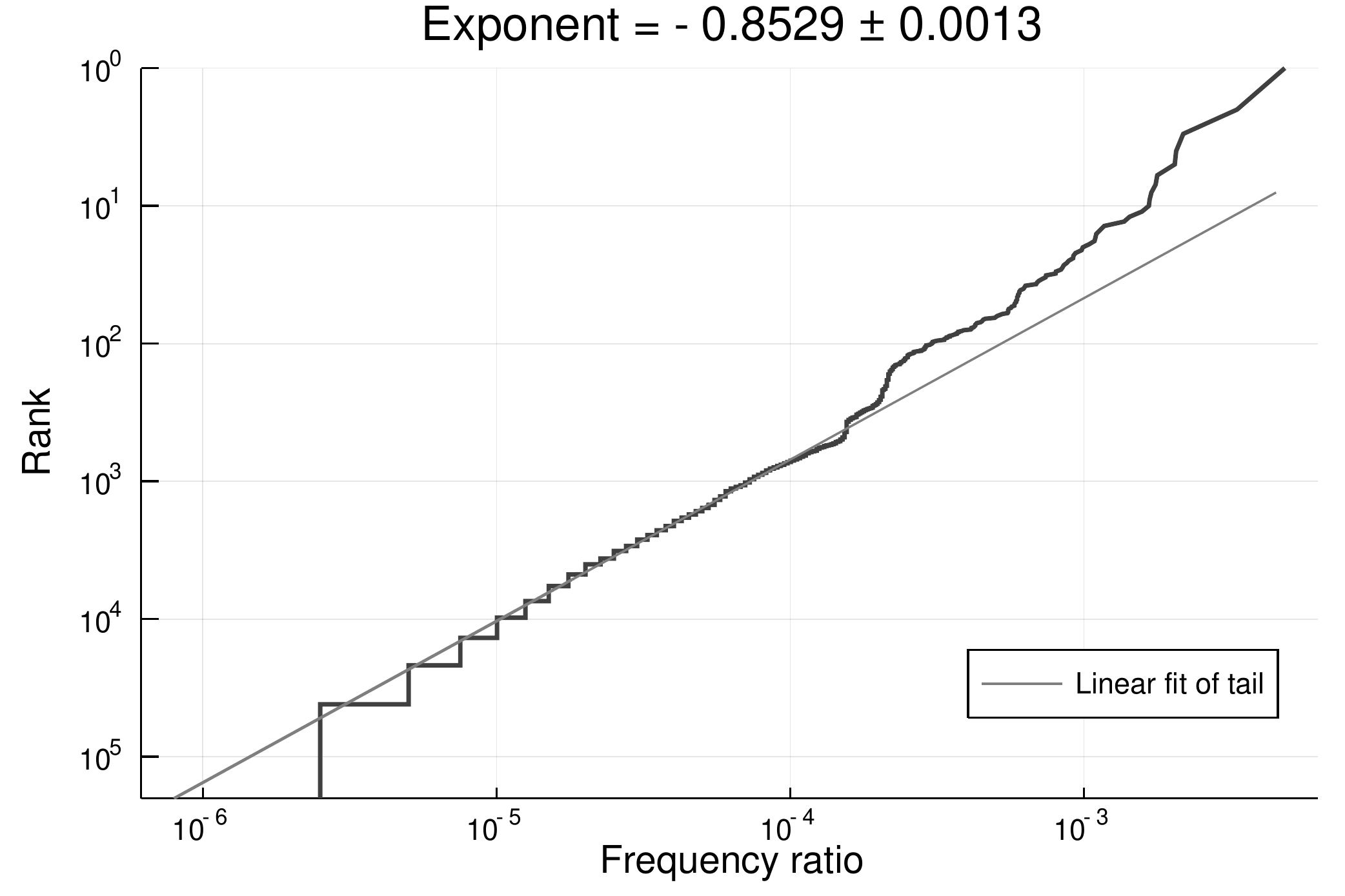}

\includegraphics[keepaspectratio=true,width=0.45\linewidth,height=\textheight]{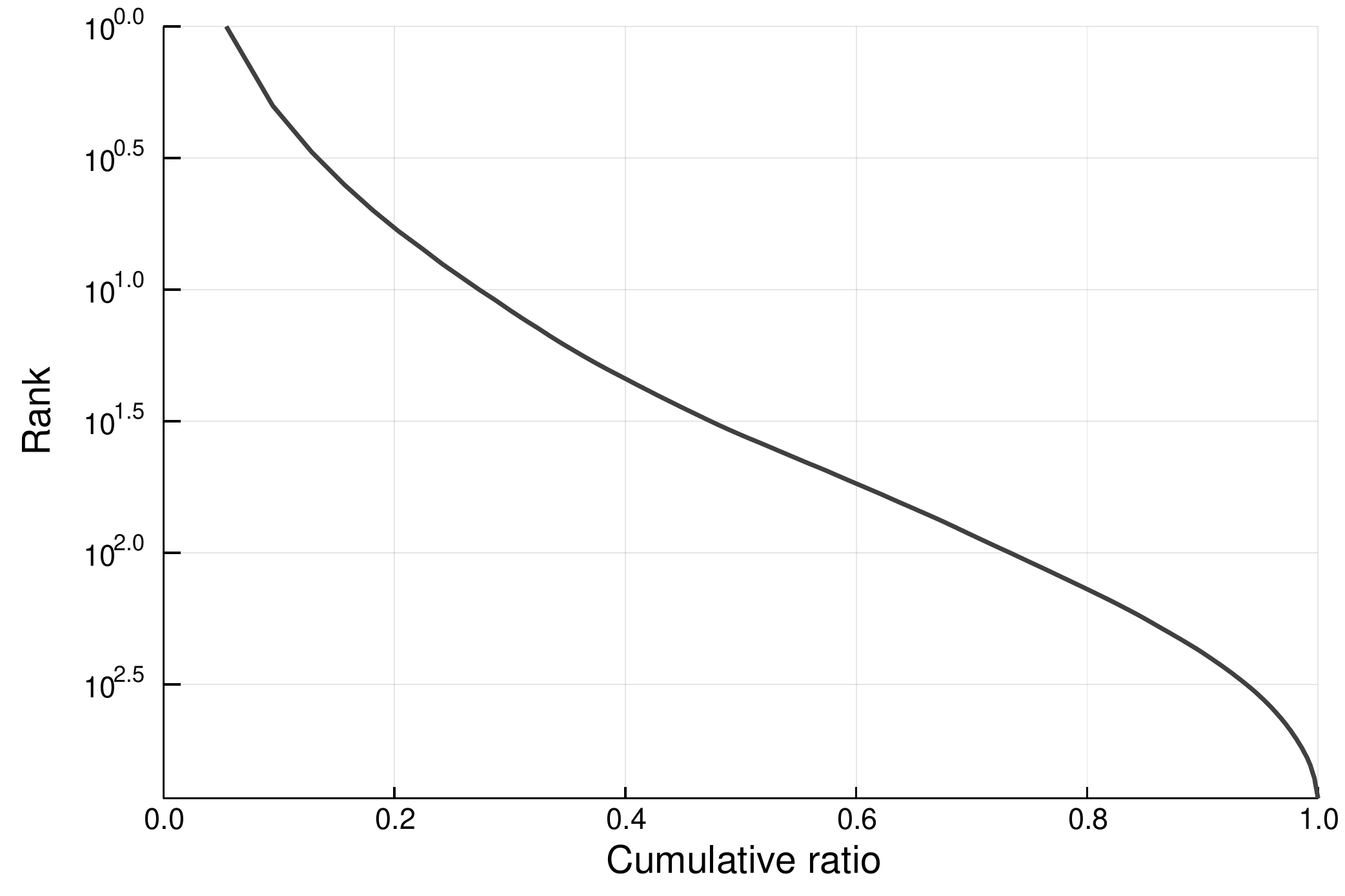}%
\includegraphics[keepaspectratio=true,width=0.45\linewidth,height=\textheight]{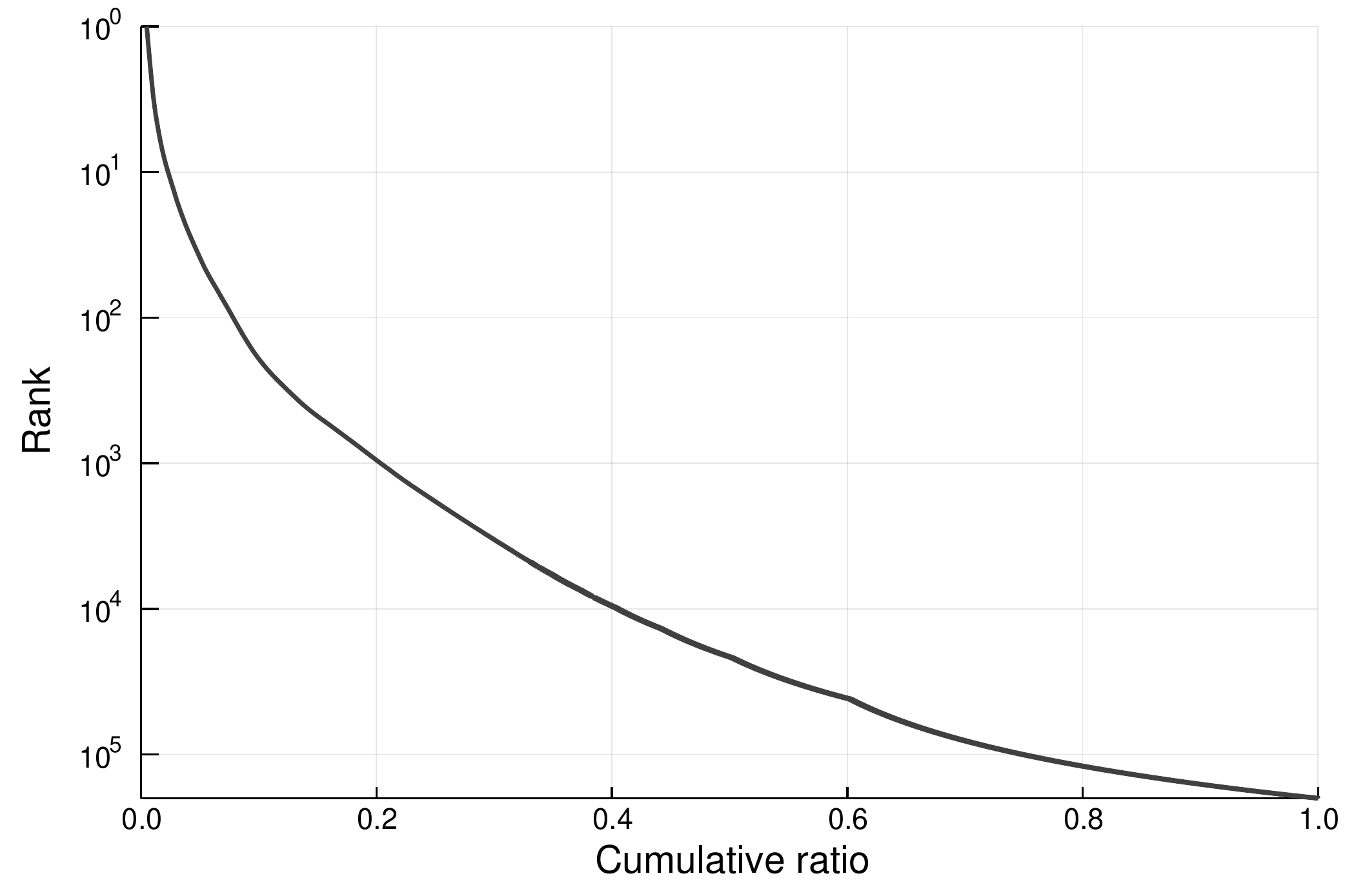}

\caption{Exploratory analysis of the Kaggle dataset.
Top: top 10 most frequent tags (left) and features (right).
Middle: rank--frequency distribution of tags and features.
Bottom: cumulative distributions of tags and features.}
\label{fig:eda}
\end{figure}

\subsection{Topics}

\Cref{fig:topic-banking} shows three typical topics for three different tags
of different frequencies of usage: ``finance'' (used 655 times), ``astronomy'' (used 70 times),
and ``cardiology'' (used 2 times).
For each topic, we show the top five most important and bottom five least important features (first row),
and also the the latent strengths (coefficients) of co-occurring with other tags (second row).
For example, the ``finance'' tag correlates positively with the ``law and government'' tag,
as well as the column names ``year'', ``project\_is\_approved'' and ``project\_title'',
and is negatively correlated with tags ``law'' and ``universities and colleges'',
as well as the column name ``day'', ``sex'' and ``target''.

Remarkably, the model is able to yield meaningful results using a completely unsupervised technique that has no conception of natural language or semantics.
Furthermore, the least important features also semantically very generic, such as the filename ``test.csv'' or column ``date''.
In addition, the topics seem to associate with meaningful features, such as variable names suggestive of particle accelerations for the ``astronomy'' topic.
The results suggest that even rare topics can be learned with meaningful signal,
which would not have been the case using a more naive approach of training an independent classifier for each topic.
With such a long tail of rare topics shown in \Cref{fig:eda}, most of the classifiers would have little meaningful training data to learn from.

\begin{figure}
\centering
\includegraphics[keepaspectratio=true,width=0.45\linewidth,height=\textheight]{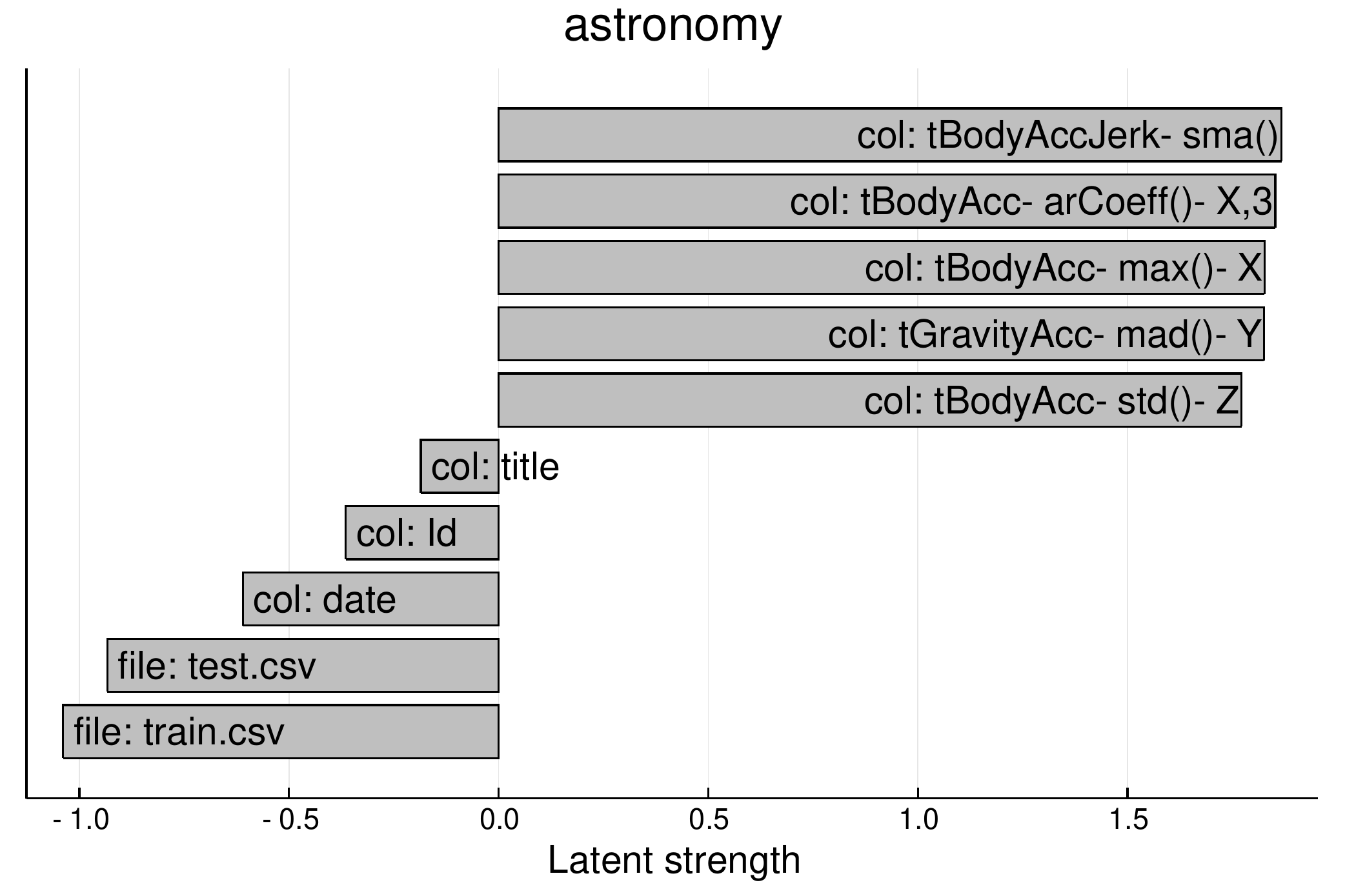}%
\includegraphics[keepaspectratio=true,width=0.45\linewidth,height=\textheight]{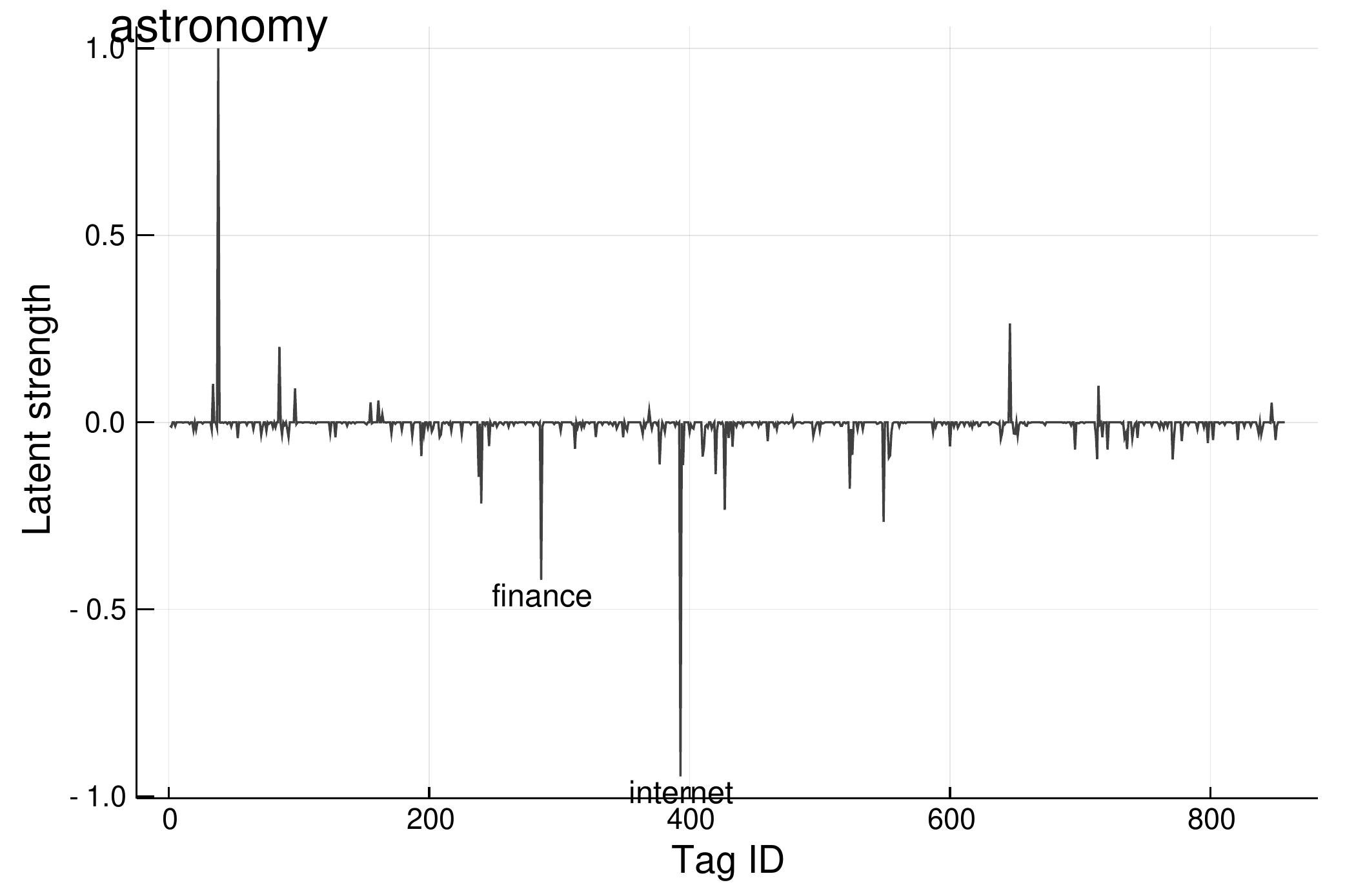}
\includegraphics[keepaspectratio=true,width=0.45\linewidth,height=\textheight]{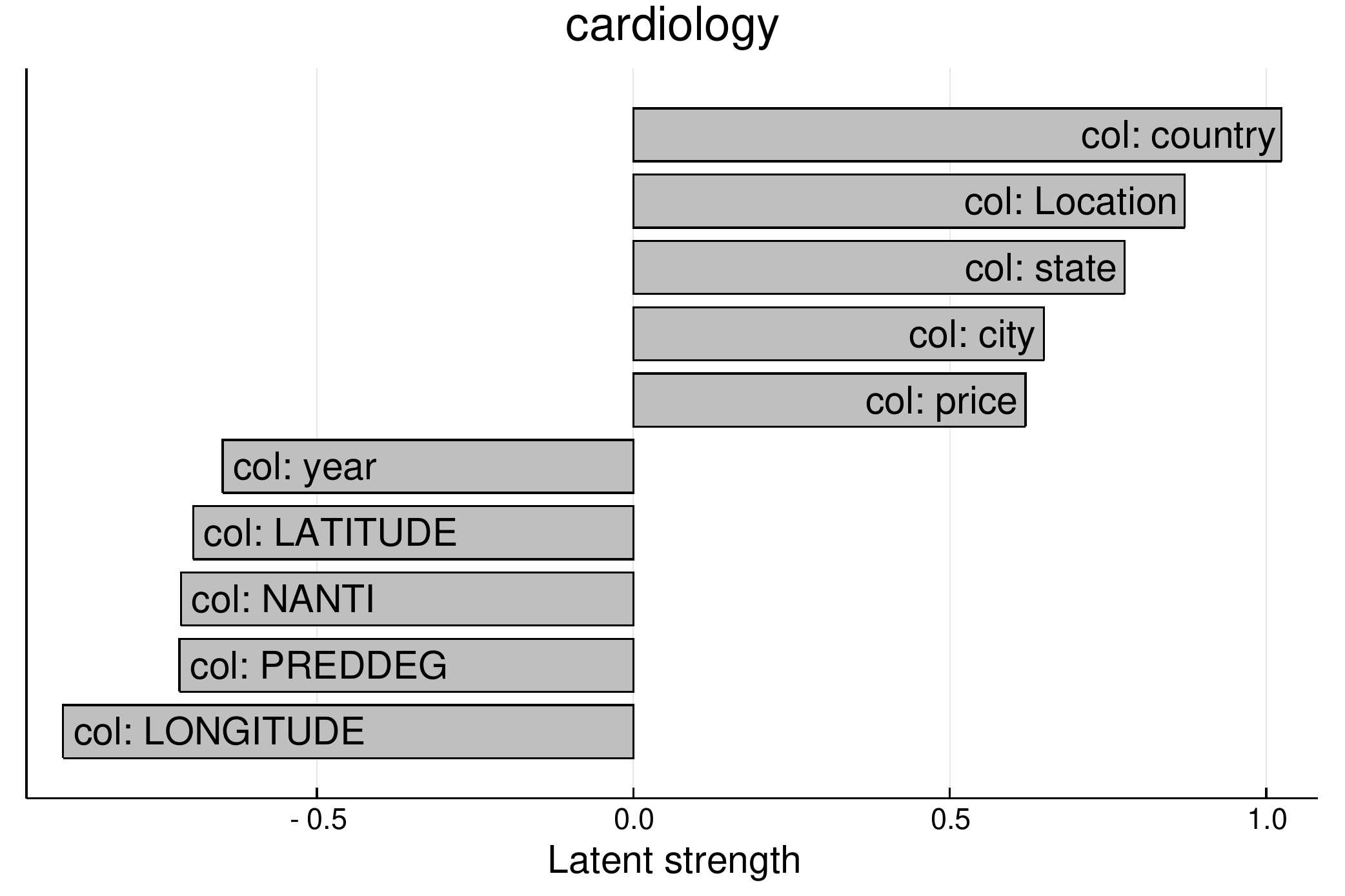}%
\includegraphics[keepaspectratio=true,width=0.45\linewidth,height=\textheight]{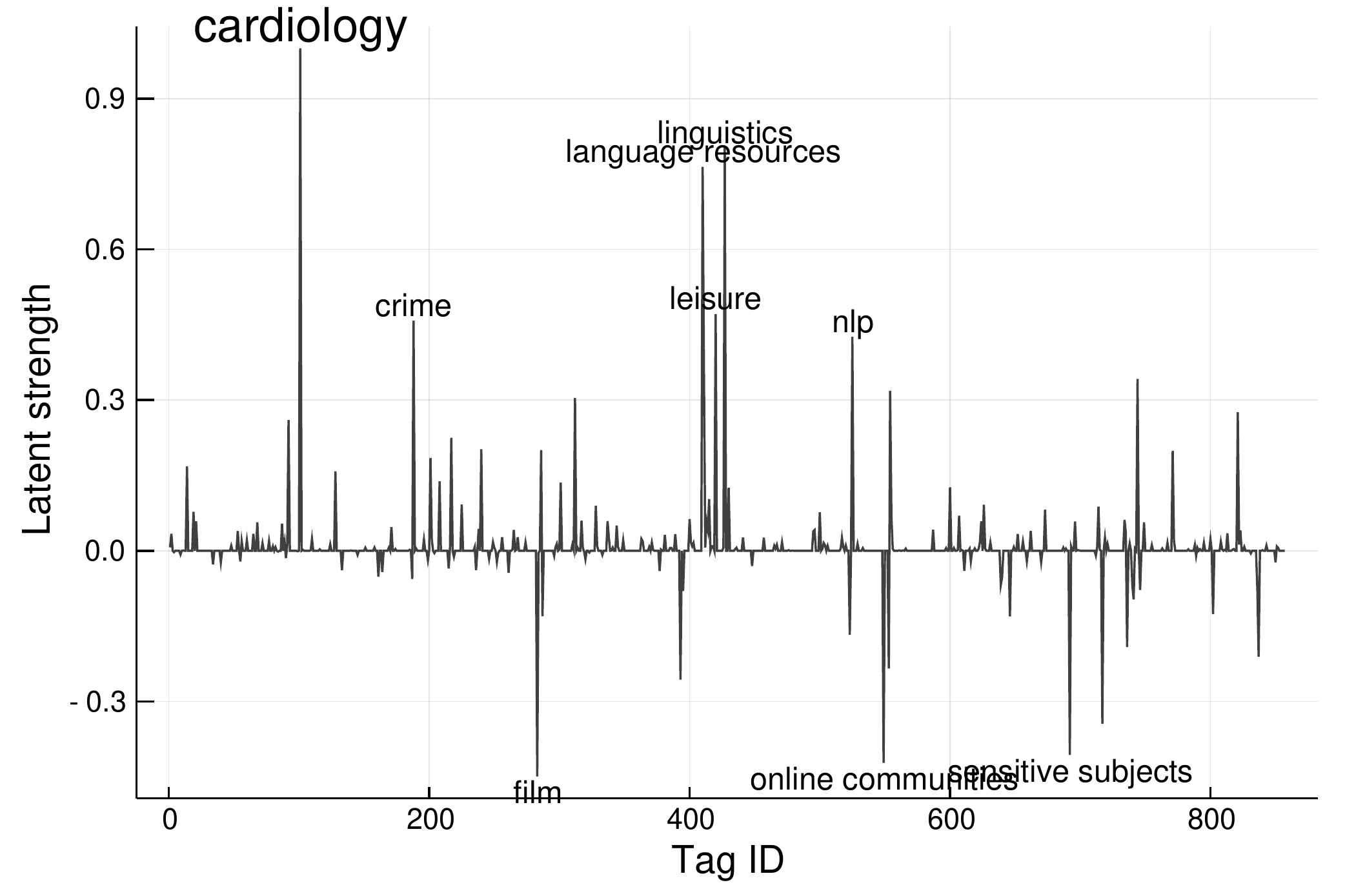}
\includegraphics[keepaspectratio=true,width=0.45\linewidth,height=\textheight]{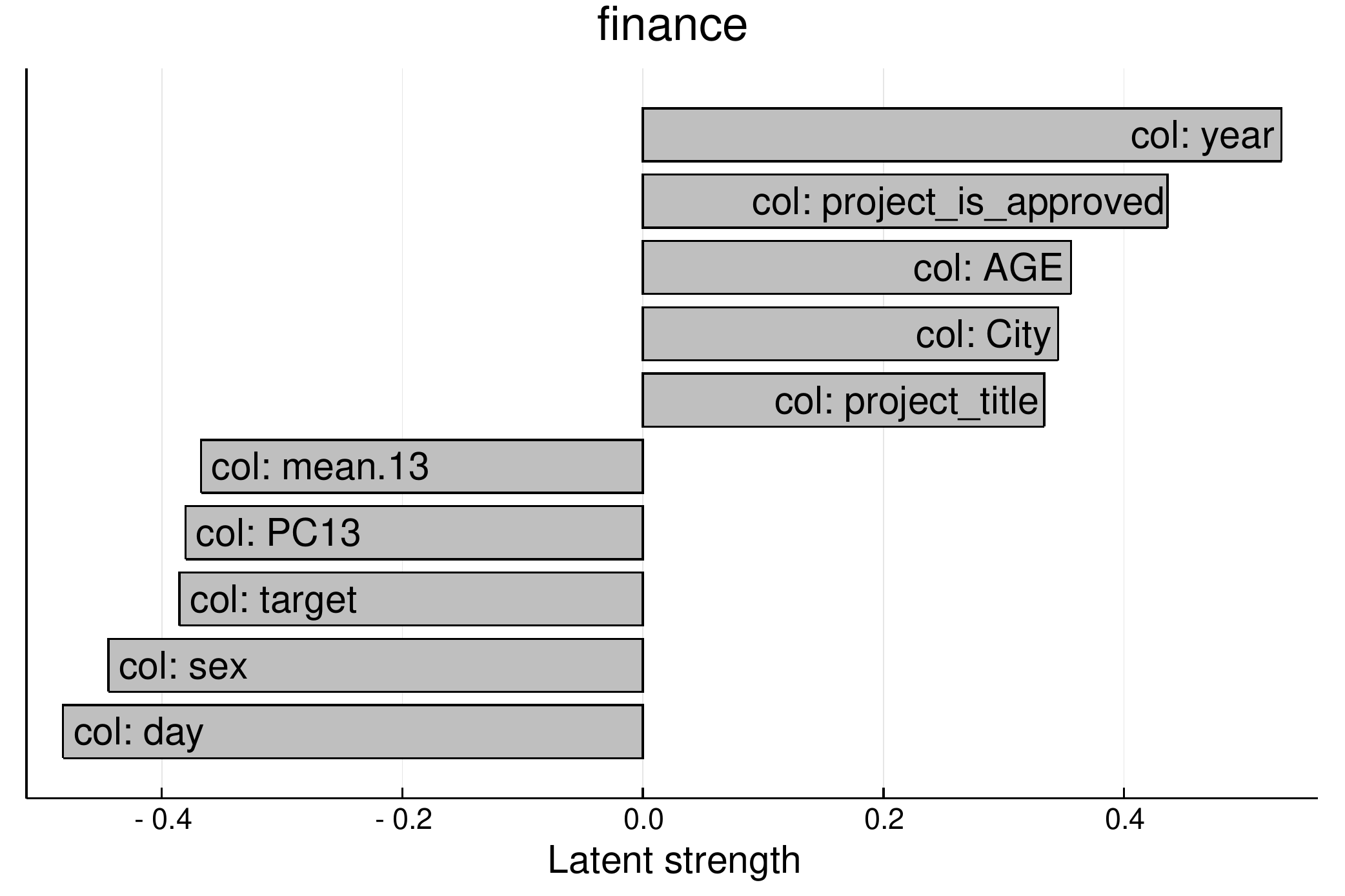}%
\includegraphics[keepaspectratio=true,width=0.45\linewidth,height=\textheight]{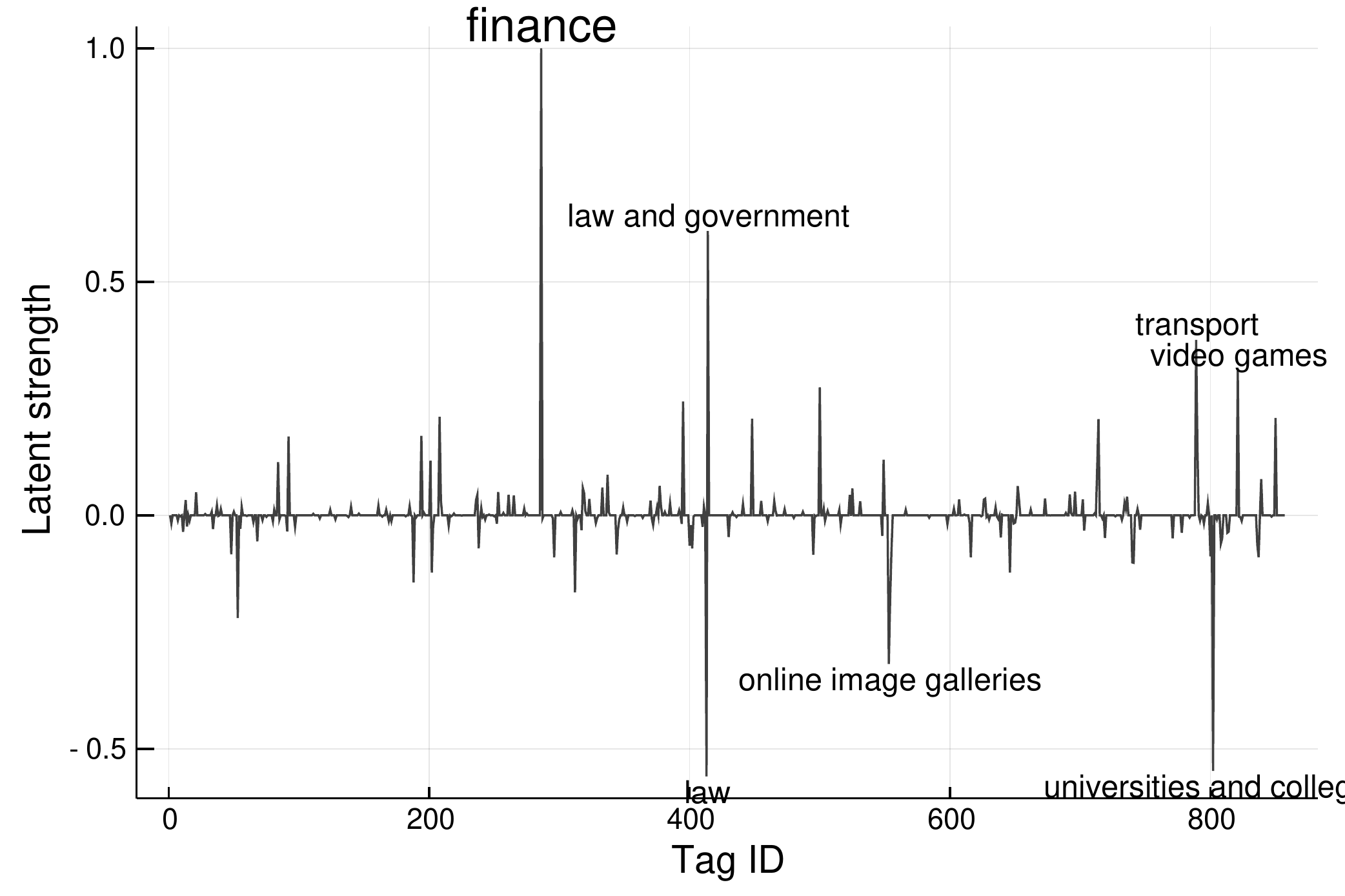}
\caption{Three representative topics from the GLRM topic model in the Kaggle data set: astronomy (top), cardiology (middle) and finance (bottom).
Left: top 5 features and bottom 5 features that are most strongly correlated with the topic.
Right: Tags that co-occur with the anchor tag for that topic.
}
\label{fig:topic-banking}
\end{figure}

We now describe two additional aspects of our study on the Kaggle data set,
the details of which we omit for brevity.

\paragraph{Sparsity in the topic model}
Perhaps unsurprisingly, the topic model trained with no sparsity penalty ($\lambda = 0$) resulted in a completely dense topic model.
Other values of $\lambda = 10^n$, $n = -3, ..., +3$ were also studied.
We observed a large increase in sparsity in going from $\lambda = 0.01$ to $\lambda = 0.1$; the results shown here are presented for
$\lambda = 0.01$. For large values of $\lambda \ge 10$, the results became effectively nonsensical, with no propensity for label
co-occurrence, but also essentially no features associated with the topic either.
Thus, the regularization parameter $\lambda$ should be chosen carefully to strike the right balance between sparsity promotion and
yielding meaningful results, perhaps using a cross-validation or other hyperparameter tuning strategy.

\paragraph{Non-negativity constraints}
We also explored a variant of the GLRM-based topic model with non-negativity constraints,
which are straightforward to incorporate using the appropriate proximal operators \citep{Udell2016}.
Constraining $Y$ to be non-negative yielded results to what we see for both the toy example of \Cref{sec:example2}
and the Kaggle data set, with the primary effect of zeroing out uninformative features (rather than assigning them negative weights of evidence in the topic).
Constraining $X$ to be non-negative was numerically unstable, yielding $Y$ matrices with large negative and large positive entries.
The training loss decreased very slowly in later iterations, suggesting very flat local curvature in the loss surface.
These results suggest that allowing both positive and negative entries in $X$ is critical to yield a well-defined topic model,
whereas nonnegativity in $Y$ may have some benefit in retraining only the features with positive weights of evidence.
This phenomenon may be worth studying in future work.

We conclude this section by discussing limitations and potential extensions of our current study.

\paragraph{Polysemy}
As described in \Cref{sec:gglsi,sec:glrm}, our current study considers only perfectly monosemic labelled topic models (\Cref{def:1topicmodel}).
A straightforward approach for handling polysemy would simply be to
identify the topics with the largest individual losses,
introduce additional topics that share the same anchor label,
and retrain the model.

\paragraph{Missing concepts}
Missing concepts can be inferred 
by adding additional columns and rows to $\datatagmatrix$ that are entirely missing,
and seeing if the resulting topic model predicts a meaningful distribution for the new topic.
However, expert judgment is still required to evaluate the hypothesized missing concepts.

\paragraph{Taxonomy learning}
The topic model does not directly yield any taxonomic structure between the concepts.
However, such a taxonomy can be learned using subsumption or hierarchical clustering methods \cite{deKnijff2013}.

\section{Conclusions}

We have introduced a topic modeling approach to learn the mapping between data concepts and
their logical representations in the form of physical metadata such as table names and column names.
Our main contribution is to introduce the GLRM-based topic model in \Cref{alg:glrm},
which makes use of a gauge-transformation approach (as illustrated on LSI in \Cref{sec:gglsi})
to define perfectly monosemic labelled topic models as a natural extension of unsupervised topic modeling.
The GLRM formulation naturally accommodates incorrect and missing labels,
as well as hypothesized sparsity of the conceptual--logical mapping,
while being also straightforward to learn at scale using randomized subsampling techniques.

We have shown using a simple example in \Cref{sec:example,sec:example2} that the results
are readily interpretable.
At the same time, the method is sufficiently scalable to analyze the metadata of Kaggle.com,
with results in \Cref{sec:kaggle} that show evidence for learning semantically meaningful topics
with interpretable features as well as relationships between data concepts.
The exact sparsity parameter $\lambda$ should be chosen carefully to yield sparse, interpretable topics.
While our results ignore the challenges of polysemy, missing concepts, and taxonomy learning,
the further extensions are all theoretically possible and can be studied in future work.


{\small
\begin{spacing}{0.9}
\paragraph{Disclaimer}
This paper was prepared for informational purposes by the Artificial Intelligence Research group of JPMorgan Chase \& Co and its affiliates (``JP Morgan''), and is not a product of the Research Department of JP Morgan.  JP Morgan makes no representation and warranty whatsoever and disclaims all liability, for the completeness, accuracy or reliability of the information contained herein.  This document is not intended as investment research or investment advice, or a recommendation, offer or solicitation for the purchase or sale of any security, financial instrument, financial product or service, or to be used in any way for evaluating the merits of participating in any transaction, and shall not constitute a solicitation under any jurisdiction or to any person, if such solicitation under such jurisdiction or to such person would be unlawful.
\end{spacing}
}
\bibliographystyle{ACM-Reference-Format}
\bibliography{bib}

\end{document}